\documentclass{article}

\usepackage{times}
\usepackage{graphicx} 
\usepackage{subfigure} 

\usepackage{natbib}

\usepackage{algorithm}
\usepackage{algorithmic}

\usepackage{hyperref}

\usepackage{amsmath}
\usepackage{amsfonts}
\usepackage{amssymb}
\usepackage{wrapfig}

\usepackage{enumitem}

\usepackage{amsthm}
\newtheorem{lem}{Lemma}

\usepackage[accepted]{icml2016}

\icmltitlerunning{Nonlinear Statistical Learning with Truncated Gaussian Graphical Models}

\begin{document} 
	
	\twocolumn[
	\icmltitle{Nonlinear Statistical Learning with Truncated Gaussian Graphical Models}
	
	\icmlauthor{Qinliang Su, Xuejun Liao, Changyou Chen, Lawrence Carin}{QS15, XJLIAO, CC448, LCARIN@duke.edu}
	\icmladdress{Duke University, Durham, NC 27519, USA}
	
	
	\vskip 0.3in
	]
	
	\begin{abstract}
		We introduce the \emph{truncated Gaussian graphical model (TGGM)} as a novel framework for designing statistical models for nonlinear learning. A TGGM is a Gaussian graphical model (GGM) with a subset of variables truncated to be nonnegative. The truncated variables are assumed latent and integrated out to induce a marginal model. We show that the variables in the marginal model are non-Gaussian distributed and their 
		expected relations are nonlinear. We use expectation-maximization to break the inference of the nonlinear model into a sequence of TGGM inference problems, each of which is efficiently solved by using the properties and numerical methods of multivariate Gaussian distributions. 
		We use the TGGM to design models for nonlinear regression and classification, with the performances of these models demonstrated on extensive benchmark datasets and compared to state-of-the-art competing results.
	\end{abstract}
	
	\section{Introduction}
	Graphical models, which use graph-based visualization to represent statistical dependencies among random variables, have been widely used to construct multivariate statistical models \cite{koller2009probabilistic}. A sophisticated model can generally represent richer statistical dependencies, but the inference can quickly become intractable as the model's complexity increases. A simple model, on the contrary,  is easy to infer, but its representational power is limited. 
	
	To balance representational versatility and inferential tractability, latent variables are often added into the graphical model to obtain a tractable joint probability distribution which, when the latent variables are integrated out, induces a complicated and expressive marginal distribution over the target variables, i.e., the variables of interest \cite{galbraith2002analysis}. Since the complexity of these models is induced by integration, expectation-maximization (EM) \cite{EM77Dempster} can be employed to facilitate inference. The approach of EM is to break the original problem of inferring the marginal distribution into a sequence of easier problems, each of which is to infer the expected logarithmic  joint distribution, where the expectation is taken over the terms of latent variables in the logarithmic domain, using the information from the previous iteration of this sequential procedure. 
	The restricted Boltzmann machine (RBM) \cite{hinton2002training} and sigmoid belief network (SBN) \cite{neal1992connectionist, gan2015learning}, as well as the deep networks built upon them \cite{salakhutdinov2009deep,hinton2006fast}, are good examples of using latent variables to enhance modeling versatility while at the same time admitting tractable statistical inference.
	
	Gaussian graphical models (GGM) constitute a subset of graphical models that have found successful application in a diverse range of areas \cite{honorio2009sparse,liu2013learning,oh2014inference, meng2014learning, su2015convergence, su2015distributed}.  The popularity of GGM may partially be attributed to the abundant applications for which the data are Gaussian distributed or approximately so, and partially attributed to the attractive properties of the multivariate Gaussian distribution which facilitate inference. However, there are many applications where the data are distributed in a way that heavily deviates from Gaussianity, and the GGM may not reveal meaningful statistical dependencies underlying the data. 
	
	What is worse, adding latent variables into a GGM does not induce enhanced marginal versatility for the target variables, as it typically does in other graphical models; this is so because the marginals of a multivariate Gaussian distribution are still Gaussian. In addition, the conditional mean of $\mathbf{y}$ given $\mathbf{x}$ is always linear in $\mathbf{x}$ whenever $(\mathbf{y},\mathbf{x})$ are jointly Gaussian. In this sense, a GGM is inherently a linear model no matter how many latent variables are added. 
	
	To overcome the linearity of GGMs, \citeauthor{frey1997continuous} (\citeyear{frey1997continuous}), \citeauthor{Hinton1997rectifierGaussian} (\citeyear{Hinton1997rectifierGaussian}), and \citeauthor{frey1999variational} (\citeyear{frey1999variational}) proposed to apply nonlinear transformations to Gaussian hidden variables.  More recently, a deep latent Gaussian model was proposed in \cite{rezende2014stochastic}, in which each hidden layer is connected to the output layer through a neural network with Gaussian noise. In these models, non-linear transforms are applied to Gaussian variables to obtain nonlinearity at the output layer. The non-linear transforms, however, destroy the nice structure of a GGM, such as the quadratic energy function and a simple conditional dependency structure, rendering it difficult to obtain analytic learning rules and, as a result, one has to resort to less efficient sampling-based inference.
	
	In this paper, we introduce a novel approach to inducing nonlinearity in a GGM. The new approach is simple: it adds latent variables into a GGM and truncates them below zero so that they are nonnegative. We term the resulting framework as \emph{truncated Gaussian graphical model (TGGM)}. Although simple, the truncation leads to a remarkable result: after the truncated latent variables are integrated out, the target variables are no longer Gaussian distributed and their expected relations are no longer linear. Therefore the TGGM induces a nonlinear marginal model for the target variables, forming a striking contrast to the GGM. It should be emphasized that the approach proposed here is different from those in \cite{socci1998rectified, downs1999nonnegative}, which constrain the target (observed) variables to be nonnegative, without using latent variables; the nonnegative constraint in those approaches relaxes the convex function in the Gaussian distribution to a non-convex energy function that admits multimodal distributions.
	
	The foremost advantage of the TGGM-induced nonlinear model over previous nonlinear models is the ease and efficiency with which inference can be performed. The advantage is attributed to the following two facts. First, as the nonlinear model is induced from a TGGM by integrating out the latent variables, EM can be used to break the inference into a sequence of TGGM inference problems. Second, as the truncation in a TGGM does not alter the quadratic energy function or the conditional dependency structure of the GGM, it is possible for a TGGM inference algorithm to utilize many well-studied properties of multivariate Gaussian distributions and the associated numerical methods \cite{johnson1994continuous,genz2009computation}. A second important advantage is that the conditional dependency structure of a TGGM is uniquely encoded in the precision matrix (or inverse covariance matrix) of the corresponding GGM (before the truncation is performed). By working with the precision matrix, one can conveniently design diverse structures and construct abundant types of nonlinear statistical models to fit the application at hand.   
	
	We provide several examples of leveraging the TGGM to solve machine-learning problems. In the first, we use the TGGM to construct a nonlinear regression model that can be understood as a probabilistic version of the rectified linear unit (ReLU) neural network \cite{glorot2011deep}. In the second, 
	we solve multi-class classification by using the multinomial probit link function \cite{albert1993bayesian} to transform the continuous target variables of a TGGM into categorical variables. Our main focus in the first two examples is on shallow structures, with one latent (hidden) layer of nonlinear units used in a TGGM. In the third example, we consider extensions to deep structures, by modifying the blocks in the precision matrix that are related to latent truncated variables. We use EM as the primary inference method, with the variational Bayesian (VB) approximation used for multivariate truncated Gaussian distributions. The performances of the TGGM models are demonstrated on extensive benchmark datasets and compared to state-of-the-art competing results. 
	\vspace{-3mm}
	\section{Nonlinearity from Truncated Gaussian Graphical Models (TGGMs)}
	Let $\mathbf{y}\in\mathbb{R}^n$ and $\mathbf{h}\in\mathbb{R}^m$ respectively denote the target and latent variables of a TGGM, and we use ${\mathbf{x}}$ to denote the input variable. The TGGM is defined by the following joint probability density function
	\vspace{-2mm}
	\begin{eqnarray} \label{joint_pdf_y_h_x}
	&&\hspace{-1cm}p({\mathbf{y}}, {\mathbf{h}}\left|{\mathbf{x}}\right.) 
	\cr&&\hspace{-1cm}={\mathcal{N}}\left({\mathbf{y}} \left| {\mathbf{W}}_1{\mathbf{h}} \!+\! {\mathbf{b}}_1, {\mathbf{P}}_1^{-1} \right. \!\right){\mathcal{N}}_T\! \left({\mathbf{h}}\left|{\mathbf{W}}_0{\mathbf{x}}\!+\! {\mathbf{b}}_0, {\mathbf{P}}_0^{-1} \right.\! \right) ,
	\end{eqnarray}
	where ${\mathcal{N}}(\mathbf{x}\left|{\boldsymbol{\mu}}, {\mathbf{P}}^{-1} \right.)$ is a multivariate Gaussian density of $\mathbf{x}$ with mean $\boldsymbol{\mu}$ and precision matrix $\mathbf{P}$ and ${\mathcal{N}_T}(\mathbf{x}\left|{\boldsymbol{\mu}}, {\mathbf{P}}^{-1} \right.)$ is the associated truncated density  defined as 
		\begin{eqnarray*}
\mathcal{N}_T(\mathbf{x}\left|\boldsymbol{\mu}, \mathbf{P}^{-1} \right.)\triangleq\frac{\mathcal{N}(\mathbf{x}\left|\boldsymbol{\mu}, \mathbf{P}^{-1} \right.)\mathbb{I}(\mathbf{x}\geq\mathbf{0})}{\int_{0}^{\infty}\mathcal{N}(\mathbf{z}\left|\boldsymbol{\mu}, \mathbf{P}^{-1} \right.)d\mathbf{z}},
\end{eqnarray*}
where $\mathbb{I}(\cdot)$ is an indicator function and $\int_0^{+\infty}d\mathbf{z}$ is  multiple integral. 
	The marginal TGGM model is defined by 
	\begin{align} \label{cond_pdf_y_x}
	&p({\mathbf{y}}\left.|{\mathbf{x}} \right.) = \int_0^{+\infty}{p({\mathbf{y}}, {\mathbf{h}}\left|{\mathbf{x}}\right.)d{\mathbf{h}}}.
	\end{align}
	To see how the truncation ${\mathbf{h}}\ge {\mathbf{0}}$ affects the marginal TGGM $p({\mathbf{y}}\left| {\mathbf{x}} \right. )$, we rewrite \eqref{joint_pdf_y_h_x} equivalently as
	\begin{eqnarray} \label{joint_pdf_y_h_x_reform}
	&&\hspace{-1.2cm}{p}({\mathbf{y}}, {\mathbf{h}} \left| {\mathbf{x}} \right.)
	\cr&&\hspace{-0.8cm}=\frac{\mathcal{N}({\mathbf{h}} |{\boldsymbol{\mu}}_{{\mathbf{h}}|{\mathbf{x}},{\mathbf{y}}}, {\boldsymbol{\Sigma}}_{{\mathbf{h}}|{\mathbf{x}},{\mathbf{y}}})
		{\mathcal{N}}({\mathbf{y}} | {\boldsymbol{\mu}}_{{\mathbf{y}}|{\mathbf{x}}}, {\boldsymbol{\Sigma}}_{{\mathbf{y}}|{\mathbf{x}}} ) \mathbb{I}({\mathbf{h}} \!\ge\! {\mathbf{0}})}{\int_0^{+\infty}{\mathcal{N}}\! \left({\mathbf{z}}\left|{\mathbf{W}}_0{\mathbf{x}}\!+\! {\mathbf{b}}_0, {\mathbf{P}}_0^{-1} \right.\! \right)d\mathbf{z}},
	\end{eqnarray}
	where  ${\boldsymbol{\mu}}_{{\mathbf{y}}|{\mathbf{x}}} = {\mathbf{W}}_1({\mathbf{W}}_0{\mathbf{x}} + {\mathbf{b}}_0)+{\mathbf{b}}_1$, ${\boldsymbol{\Sigma}}_{{\mathbf{y}}|{\mathbf{x}}} = {\mathbf{W}}_1{\mathbf{P}}_0^{-1}{\mathbf{W}}_1 + {\mathbf{P}}_1^{-1}$, ${\boldsymbol{\mu}}_{{\mathbf{h}}|{\mathbf{x}},{\mathbf{y}}}=({\mathbf{P}}_0+{\mathbf{W}}_1^T{\mathbf{P}}_1{\mathbf{W}}_1)^{-1}({\mathbf{P}}_0({\mathbf{W}}_0{\mathbf{x}}+{\mathbf{b}}_0)+{\mathbf{W}}_1^T{\mathbf{P}}_1({\mathbf{y}}-{\mathbf{b}}_1))$, and ${\boldsymbol{\Sigma}}_{{\mathbf{h}}|{\mathbf{x}},{\mathbf{y}}}=({\mathbf{P}}_0+{\mathbf{W}}_1^T{\mathbf{P}}_1{\mathbf{W}}_1)^{-1}$. 
	From \eqref{joint_pdf_y_h_x_reform} and \eqref{cond_pdf_y_x} follows
	\begin{eqnarray} \label{p_y_x_modul}
	&&\hspace{-1cm}p({\mathbf{y}}\left| {\mathbf{x}}\right.)
	\vspace{-3mm}
	\cr&&\hspace{-1cm}={\mathcal{N}}({\mathbf{y}} | {\boldsymbol{\mu}}_{{\mathbf{y}}|{\mathbf{x}}}, {\boldsymbol{\Sigma}}_{{\mathbf{y}}|{\mathbf{x}}} )\frac{\int_0^{+\infty}\mathcal{N}({\mathbf{h}} |{\boldsymbol{\mu}}_{{\mathbf{h}}|{\mathbf{x}},{\mathbf{y}}}, {\boldsymbol{\Sigma}}_{{\mathbf{h}}|{\mathbf{x}},{\mathbf{y}}})d\mathbf{h}}{\int_0^{+\infty}{\mathcal{N}}\! \left({\mathbf{h}}\left|{\mathbf{W}}_0{\mathbf{x}}\!+\! {\mathbf{b}}_0, {\mathbf{P}}_0^{-1} \right.\! \right)d\mathbf{h}}.
	\end{eqnarray}
	It is seen from \eqref{p_y_x_modul} that the target distribution induced by a TGGM consists of two multiplicative terms. The first term is a Gaussian distribution  induced by the associated GGM (for which $\mathbf{h}$ is not truncated). The second term modulates the Gaussian term into a more complicated and non-Gaussian distribution. As an example, one can verify that  (\ref{p_y_x_modul}) is a skewed normal with $w_0=b_0=b_1=0$, $p_0=1$, $p_1=2$, and $w_1=1/2$ \cite{mudholkar2000epsilon}.

	The modeling versatility of a TGGM is primarily influenced by $m$, the number of truncated latent variables, and  $\mathbf{P}_0$, which encodes marginal dependencies of these variables. With a proper choice of $m$ and $\mathbf{P}_0$, one can construct TGGM models to solve diverse  nonlinear learning tasks.  The nonlinearity induced by a TGGM is seen from the expression of $\mathbb{E}[\mathbf{y}|\mathbf{x}]$ which, using \eqref{joint_pdf_y_h_x}, is found to be 
	\begin{equation}\label{mean_y_given_x}
	{\mathbb{E}}[{\mathbf{y}}|{\mathbf{x}}] = {\mathbf{W}}_1{\mathbb{E}}[{\mathbf{h}}|{\mathbf{x}}] + {\mathbf{b}}_1, 
	\end{equation}
	where ${\mathbb{E}}[{\mathbf{h}}|{\mathbf{x}}]$ is the expectation with respect to $ {\mathcal{N}}_T\! \left({\mathbf{h}}\left|{\mathbf{W}}_0{\mathbf{x}}\!+\! {\mathbf{b}}_0, {\mathbf{P}}_0^{-1} \right.\! \right) $. Due to the truncation $\mathbf{h}\geq0$, the expectation ${\mathbb{E}}[{\mathbf{h}}|{\mathbf{x}}]$ is a nonlinear function of ${\mathbf{x}}$. By contrast, if ${\mathbf{h}}$ is not truncated, one has ${\mathbb{E}}[{\mathbf{h}}|{\mathbf{x}}]={\mathbf{W}}_0{\mathbf{x}} + {\mathbf{b}}_0$, which is a linear function of $\mathbf{x}$. Thus, a TGGM induces nonlinearity through the truncation of its latent variables.
	
	The nonlinearity can be controlled by adjusting  ${\mathbf{P}}_0$. For example, if we set ${\mathbf{P}}_0=\frac{1}{\sigma^2}{\mathbf{I}_m}$, where  $\mathbf{I}_m$ is a $m\times{}m$ identity matrix, we obtain 
	\begin{equation}\label{mean_y_given_x-diagonal}
	{\mathbb{E}}[{\mathbf{h}}(k)|{\mathbf{x}}] = g\left(\mathbf{W}_0(k, :)\mathbf{x} + \mathbf{b}_0(k), \sigma\right), 
	\end{equation}
	where ${\mathbf{h}}(k)$ is the $k$-th element of ${\mathbf{h}}$ and ${\mathbf{W}}_0(k,:)$ the $k$-th row of ${\mathbf{W}}_0$ using Matlab notations, and $g(\mu,\sigma)$ is the mean of the univariate truncated normal distribution ${\mathcal{N}}_T(x\left|\mu, \sigma^2\right.)$. The formula of $g(\mu,\sigma)$ is given by \cite{johnson1994continuous}
		\vspace{1mm}
\begin{eqnarray}
g(\mu, \sigma) \triangleq \mu + \sigma \frac{\phi\left(\frac{\mu}{\sigma}\right)}{\Phi\left(\frac{\mu}{\sigma}\right)},\label{truncated_mean_general}
\end{eqnarray}
where ${\phi(z)} \triangleq \frac{1}{\sqrt{2\pi}}\exp^{-\frac{z^2}{2}}$ is the probability density function (PDF) of the standard normal distribution, 	and $\Phi(z) \triangleq \int_{-\infty}^{z}{\phi(t)dt}$ its cumulative distribution function (CDF). 
	Figure \ref{fig:truncated-activation} shows $g(\mu,\sigma)$ as a function of $\mu$, for various values of $\sigma$, alongside $\max(0,\mu)$, which is the activation function used in ReLU neural networks \cite{glorot2011deep}. It is seen that $g(\mu, \sigma)$ is a soft version of $\max(0,\mu)$ and $\lim_{\sigma\to 0}g(\mu, \sigma)=\max(0,\mu)$. 
	\begin{figure}
		\centering
		\includegraphics[width=0.30\textwidth]{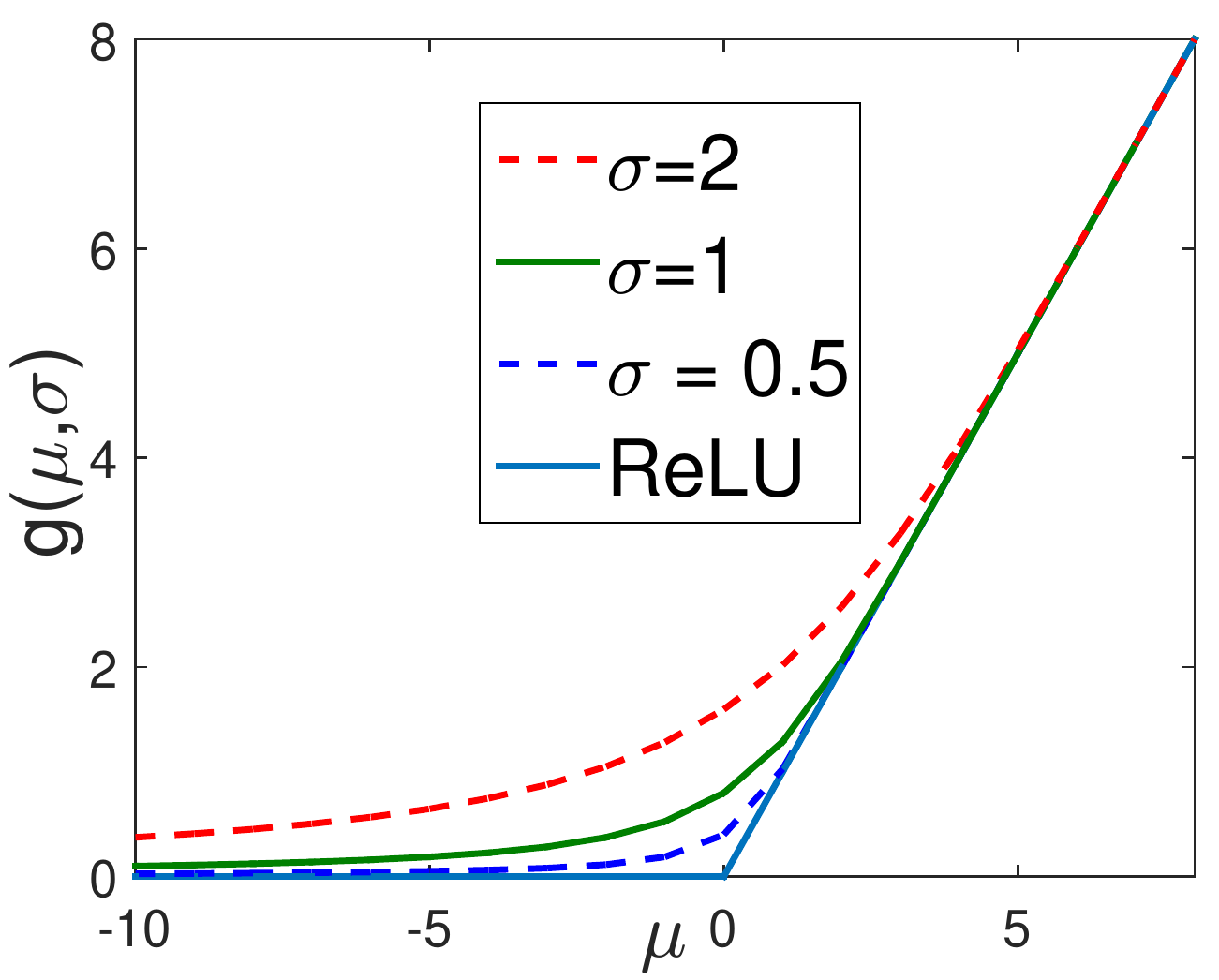}
		\caption{\label{fig:truncated-activation}Visualization of $g(\mu,\sigma)$ as a function of $\mu$, for various values of $\sigma$, in comparison to $\max(0,\mu)$.}
	\end{figure}
	
	\vspace{1mm}
	\section{Nonlinear Regression with TGGMs}\label{reg_training}
	We begin with a nonlinear regression model constructed from a simple TGGM, in which we restrict ${\mathbf{P}}_0$ and ${\mathbf{P}}_1$ to diagonal matrices: $\mathbf{P}_0=\sigma_0^2\mathbf{I}_m$ and $\mathbf{P}_1=\sigma_1^2\mathbf{I}_n$. 	By (\ref{mean_y_given_x-diagonal})-(\ref{truncated_mean_general}) and the arguments there, the $\mathbb{E}(\mathbf{y}|\mathbf{x})$ in such a TGGM implements the output of a soft-version ReLU neural network, which has a single layer of $m$ hidden units with the activation function $g(\cdot,\sigma)$, and uses $\mathbf{W}_0$ and $\mathbf{W}_1$ as the input-to-hidden and hidden-to-output weights, respectively.

	\subsection{Maximum-Likelihood (ML) Parameter Estimation} \label{sec:reg-ML}
	Given a training dataset consisting of the inputs (covariates) ${\mathbf{X}}=[{\mathbf{x}}_1, {\mathbf{x}}_2, \cdots, {\mathbf{x}}_N]$ and the outputs (responses) ${\mathbf{Y}}=[{\mathbf{y}}_1, {\mathbf{y}}_2, \cdots, {\mathbf{y}}_N]$, the log-likelihood function is 
	\begin{eqnarray}\label{eq:likelihood}
	{\mathcal{L}}({\boldsymbol{\Theta}}) \triangleq \ln \int_{0}^{\infty}{p({\mathbf{Y}}, {\mathbf{H}}|{\mathbf{X}}; {\boldsymbol{\Theta}})d{\mathbf{H}}}
	\end{eqnarray}
	where ${\boldsymbol{\Theta}} \triangleq \{{\mathbf{W}}_0, {\mathbf{W}}_1, {\mathbf{b}}_1, {\mathbf{b}}_0\}$, and 
	\begin{eqnarray}\label{eq:complete-likelihood}
	&&\hspace{-1cm}p({\mathbf{Y}}, {\mathbf{H}} \left| {\mathbf{X}}, \boldsymbol{\Theta} \right.) = \mbox{$\prod_{i=1}^N{\mathcal{N}}_T$}\left({\mathbf{h}}_i \left|  {\mathbf{W}}_0{\mathbf{x}}_i+{\mathbf{b}}_0, \sigma_0^2 {\mathbf{I}}_m \right.\right) 
	\cr&&\hspace{1.8cm}\times{\mathcal{N}}\left({\mathbf{y}}_i \left| {\mathbf{W}}_1{\mathbf{h}}_i+{\mathbf{b}}_1, \sigma^2_1{\mathbf{I}}_n\right.\right). 
	\end{eqnarray}
	Let $q(\mathbf{H}|\widetilde{\boldsymbol{\Theta}})$ be an arbitrary PDF with parameters $\widetilde{\boldsymbol{\Theta}}$, defined on $\{\mathbf{H}:\mathbf{H}\geq0\}$. It follows from (\ref{eq:likelihood})
	\begin{eqnarray}\label{eq:lower-bound}
	\hspace{-0cm}{\mathcal{L}}({\boldsymbol{\Theta}}) &=&\ln \int_{0}^{\infty}q(\mathbf{H}|\widetilde{\boldsymbol{\Theta}})\frac{p({\mathbf{Y}}, {\mathbf{H}}|{\mathbf{X}}; {\boldsymbol{\Theta}})}{q(\mathbf{H}|\widetilde{\boldsymbol{\Theta}})}d{\mathbf{H}},
	\cr\hspace{-1cm}&\geq& \int_{0}^{\infty}q(\mathbf{H}|\widetilde{\boldsymbol{\Theta}})\ln \frac{p({\mathbf{Y}}, {\mathbf{H}}|{\mathbf{X}}; {\boldsymbol{\Theta}})}{q(\mathbf{H}|\widetilde{\boldsymbol{\Theta}})}d{\mathbf{H}},
	\cr&&(\textrm{Jensen's Inequality})\nonumber
	\cr\hspace{-1cm}&=&{\mathcal{L}}({\boldsymbol{\Theta}})-\mathrm{KL}(q(\cdot|\widetilde{\boldsymbol{\Theta}})||p(\cdot|{\mathbf{Y}},{\mathbf{X}}; {\boldsymbol{\Theta}})),
	\\\hspace{-1cm}&\triangleq& \mathcal{Q}_{q(\cdot|\widetilde{\boldsymbol{\Theta}})}(\boldsymbol{\Theta}).
	\end{eqnarray}
	where $\mathrm{KL}(q(\cdot)||p(\cdot))$ denotes the Kullback-Leibler (KL) distance, and, if $\exists$ $\mathbf{H}$ such that $q(\mathbf{H})=0$, the limit values are used to lead to $\frac{q(\mathbf{H})}{q(\mathbf{H})}=1$ and $q(\mathbf{H})\ln{}q(\mathbf{H})=0$. 
	In general $q(\cdot)$ is parameterized differently from the TGGM; when $q(\mathbf{H}|\widetilde{\boldsymbol{\Theta}})=p(\mathbf{H}|\mathbf{Y},\mathbf{X},\boldsymbol{\Theta})$, however, we let $q(\cdot)$ use the same parameterization as the TGGM so that  $\widetilde{\boldsymbol{\Theta}}=\boldsymbol{\Theta}$.  In this case, we drop the subscript to  simply write $\mathcal{Q}(\boldsymbol{\Theta}|\widetilde{\boldsymbol{\Theta}})\equiv\mathcal{Q}_{p(\cdot|\mathbf{Y},\mathbf{X},\widetilde{\boldsymbol{\Theta}})}(\boldsymbol{\Theta})$, 
	which, by (\ref{eq:lower-bound}), can be further written as 
	\begin{eqnarray}\label{eq:lower-bound-EM}
	\mathcal{Q}(\boldsymbol{\Theta}|\widetilde{\boldsymbol{\Theta}})\!=\!{\mathcal{L}}({\boldsymbol{\Theta}})\!-\!\mathrm{KL}(p(\cdot|{\mathbf{Y}}\!,{\mathbf{X}}; \widetilde{\boldsymbol{\Theta}})||p(\cdot|{\mathbf{Y}}\!,{\mathbf{X}}; {\boldsymbol{\Theta}})),
	\end{eqnarray}
	From (\ref{eq:lower-bound-EM}) follows the EM algorithm. First, it is clear that $\mathcal{Q}(\boldsymbol{\Theta}|\boldsymbol{\Theta})=\mathcal{L}({\boldsymbol{\Theta}})$. Thus, for a sequence $\{\boldsymbol{\Theta}_{t}\}$ satisfying 
	\begin{eqnarray}\label{eq:E-step}
	\boldsymbol{\Theta}_{t+1}=\mathrm{arg}\max_{\boldsymbol{\Theta}}\mathcal{Q}(\boldsymbol{\Theta}|\boldsymbol{\Theta}_{t}),
	\end{eqnarray}
	one deduces $\mathcal{L}(\boldsymbol{\Theta}_{t})$ $=$ $\mathcal{Q}(\boldsymbol{\Theta}_{t}|\boldsymbol{\Theta}_{t})$ $\leq$ $\mathcal{Q}(\boldsymbol{\Theta}_{t+1}|\boldsymbol{\Theta}_{t})$ $\leq$ $\mathcal{L}(\boldsymbol{\Theta}_{t+1})$, where the last inequality follows from (\ref{eq:lower-bound-EM}). By successively solving (\ref{eq:E-step}), starting from initial $\boldsymbol{\Theta}_{1}$, the EM algorithm produces a sequence $\{\boldsymbol{\Theta}_{t}:t\geq1\}$ that monotonically increases $\mathcal{L}(\boldsymbol{\Theta}_{t})$. 
	To ensure $\mathcal{L}(\boldsymbol{\Theta}_{t+1})>\mathcal{L}(\boldsymbol{\Theta}_{t})$, one only requires $\mathcal{Q}(\boldsymbol{\Theta}_{t+1}|\boldsymbol{\Theta}_{t})$ $>$ $\mathcal{Q}(\boldsymbol{\Theta}_{t}|\boldsymbol{\Theta}_{t})$. Therefore it is not necessary to solve  (\ref{eq:E-step}) completely; rather it is sufficient to perform a single-step gradient ascent from $\boldsymbol{\Theta}_{t}$, 
	\begin{eqnarray}\label{eq:gradient-EM}
	\hspace{-0.65cm}\boldsymbol{\Theta}_{t+1}\!\!\!\!\!\!&=&\!\!\!\!\!\boldsymbol{\Theta}_{t}+\gamma_{t}\nabla_{\boldsymbol{\Theta}}\mathcal{Q}(\boldsymbol{\Theta}|\boldsymbol{\Theta}_{t})\big|_{\boldsymbol{\Theta}=\boldsymbol{\Theta}_{t}},
	\end{eqnarray}
	where $\mathcal{Q}(\boldsymbol{\Theta}|\boldsymbol{\Theta}_{t}\!)\!\!=\!\!\!\!\int_{0}^{\infty}\!\!p({\mathbf{H}}|\mathbf{Y}\!,{\mathbf{X}}; {\boldsymbol{\Theta}}_t\!)\ln{}p({\mathbf{Y}}\!, {\mathbf{H}}|{\mathbf{X}}; {\boldsymbol{\Theta}})d\mathbf{Z}$.
	To find the gradient, it is helpful to write $p({\mathbf{Y}}, {\mathbf{H}}|{\mathbf{X}}; {\boldsymbol{\Theta}}) = \frac{1}{Z({\mathbf{X}}; {\boldsymbol{\Theta}})} e^{-E({\mathbf{Y}}, {\mathbf{H}}|{\mathbf{X}}; {\boldsymbol{\Theta}})}$, where $E = \sum_{i=1}^N \frac{\left\|{\mathbf{h}}_i- {\mathbf{W}}_0{\mathbf{x}}_i\right\|^2}{2\sigma^2_0} + \sum_{i=1}^N \frac{\left\|{\mathbf{y}}_i - {\mathbf{W}}_1{\mathbf{h}}_i \right\|^2}{\sigma^2_1}$ is the energy function and $Z$ the normalization. The gradient can then be expressed as
	\begin{align} \label{derivative_two_layer}
	\nabla_{\boldsymbol{\Theta}}\mathcal{Q} = -{\mathbb{E}}\left[\left. \frac{\partial E}{\partial {\boldsymbol{\Theta}}}\right|{\mathbf{Y}}, {\mathbf{X}} \right] + {\mathbb{E}}\left[\left. \frac{\partial E}{\partial {\boldsymbol{\Theta}}}\right|{\mathbf{X}}\right],
	\end{align}
	where ${\mathbb{E}}[\cdot|{\mathbf{Y}}, {\mathbf{X}} ]$ denotes the expectation with respect to (w.r.t.)  $p({\mathbf{H}}|{\mathbf{X}}, {\mathbf{Y}}; {\boldsymbol{\Theta}}_{t})$, and ${\mathbb{E}}[\cdot|{\mathbf{X}} ]$ the expectation w.r.t. $p({\mathbf{Y}}, {\mathbf{H}}| {\mathbf{X}}; {\boldsymbol{\Theta}}_{t})=p({\mathbf{Y}}|{\mathbf{H}}; {\boldsymbol{\Theta}}_{t})\,p({\mathbf{H}}| {\mathbf{X}}; {\boldsymbol{\Theta}}_{t})$. Specifically, the partial derivatives of $\mathcal{\mathcal{Q}}$ w.r.t. ${\mathbf{W}}_0$ and ${\mathbf{W}}_1$ can respectively be derived as
	\begin{align}
	\frac{\partial {\mathcal{Q}}}{\partial {\mathbf{W}}_0} &= -\frac{1}{\sigma^2_0}\left({\mathbb{E}}[{\mathbf{H}}|{\mathbf{X}} ] - {\mathbb{E}}[{\mathbf{H}}|{\mathbf{Y}}, {\mathbf{X}}]\right){\mathbf{X}}^T,  \label{grad_W0_reg}\\
	\frac{\partial {\mathcal{Q}}}{\partial {\mathbf{W}}_1}  &= -\frac{1}{\sigma^2_1} \Big({\mathbf{W}}_1{\mathbb{E}}\!\left[{\mathbf{H}}{\mathbf{H}}^T|{\mathbf{Y}}, {\mathbf{X}} \right] \nonumber \\
	&\quad\quad\quad\quad\quad - ({\mathbf{Y}} - {\mathbf{b}}_1{\mathbf{1}}^T_N){\mathbb{E}}[{\mathbf{H}}^T|{\mathbf{Y}}, {\mathbf{X}}]\Big),  \label{grad_w1_reg}
	\end{align}
	where ${\mathbf{1}}_N$ is a column vector of ones. The derivatives w.r.t. ${\mathbf{b}}_0$ and ${\mathbf{b}}_1$ can be  derived similarly.
	
	\subsection{ML Estimation versus Backpropagation}\label{sec:reg-ML-vs-BP}
	As mentioned earlier, for a TGGM with diagonal $\mathbf{P}_0$ and $\mathbf{P}_1$, $\mathbb{E}(\mathbf{y}|\mathbf{x})$  implements the output of a soft-version ReLU neural network that use (\ref{truncated_mean_general}) as the activation function at each hidden unit. This suggests one can use back-propagation (BP) to minimize the error between $\mathbb{E}(\mathbf{Y}|\mathbf{X})$ and the training samples of $\mathbf{Y}$, as one does in training a standard ReLU network \cite{glorot2011deep}. 
	
	A popular choice of the error function used by BP is the squared error which, in the case here, is given by ${\mathcal{E}} \triangleq \frac{1}{2\sigma_1^2}\left\|{\mathbf{W}}_1\mathbb{E}(\mathbf{H}|\mathbf{X})+ {\mathbf{b}}_1{\mathbf{1}}_N^T - {\mathbf{Y}} \right\|^2$. Minimization of the squared error is equivalent to maximization of the likelihood under the assumption that $\mathbf{y}|\mathbf{x}\sim\mathcal{N}(\mathbf{y}|\mathbb{E}(\mathbf{y}|\mathbf{x}),\sigma_1^2)$. However, we have shown in (\ref{p_y_x_modul}) that $p(\mathbf{y}|\mathbf{x})$ is a non-Gaussian distribution. Therefore, BP does not maximize the likelihood of the TGGM in the rigorous sense. 
	
	To gain a deeper understanding of the relation between BP and ML estimation, we analyze the update equations of BP and compare them to those of the ML estimator. The BP performs gradient descent of the squared error, with the required partial derivatives given by 
	\begin{align}
	\hspace{-0.2cm}\frac{\partial {\mathcal{E}}}{\partial {\mathbf{W}}_0 }  &\!=\!\! -\! \left(\left({\mathbf{W}}_1^T(\mathbb{E}(\mathbf{Y}|\mathbf{X})\!\! -\! {\mathbf{Y}}) \right) \! \!\odot\! \frac{\mathrm{Var}(\mathbf{H}|\mathbf{X})}{\sigma_0^2} \right)\! {\mathbf{X}}^T \label{grad_E_W0}, 
	\\\hspace{-0.2cm}\frac{\partial {\mathcal{E}}}{\partial {\mathbf{W}}_1}  &= -\frac{1}{\sigma^2_1} \Big({\mathbf{W}}_1\mathbb{E}(\mathbf{H}|\mathbf{X}) \,\mathbb{E}(\mathbf{H}^T|\mathbf{X})  \nonumber \\
	&\hspace{2cm} - ({\mathbf{Y}} - {\mathbf{b}}_1{\mathbf{1}}^T_N)\mathbb{E}(\mathbf{H}^T|\mathbf{X}) \Big), \label{grad_E_W1}
	\end{align}
	where $\odot$ is the Hadamard product and $\mathrm{Var}(\mathbf{H}|\mathbf{X})\triangleq \mathbb{E}\Big[\big(\mathbf{H}-\mathbb{E}(\mathbf{H}|\mathbf{X})\big)\odot\big(\mathbf{H}-\mathbb{E}(\mathbf{H}|\mathbf{X})\big)\Big|\mathbf{X}\Big]$ is a matrix of variances. Comparing \eqref{grad_E_W1} to \eqref{grad_w1_reg}, we can see that the direction of $\frac{\partial {\mathcal{E}}}{\partial {\mathbf{W}}_1 } $ is an approximation to that of $\frac{\partial {\mathcal{Q}}}{\partial {\mathbf{W}}_1 }=-\frac{1}{\sigma^2_1}\!\! \left({\mathbf{W}}_1{\mathbb{E}}\!\left[{\mathbf{H}}{\mathbf{H}}^T|{\mathbf{Y}}, {\mathbf{X}}\right] \!\!-\! ({\mathbf{Y}} \!\!-\! {\mathbf{b}}_1{\mathbf{1}}^T_N){\mathbb{E}}[{\mathbf{H}}^T|{\mathbf{Y}}, {\mathbf{X}}]\right)$ by replacing the posterior expectations ${\mathbb{E}}[{\mathbf{H}}{\mathbf{H}}^T|{\mathbf{Y}}, {\mathbf{X}}]$ and ${\mathbb{E}}[{\mathbf{H}}|{\mathbf{Y}}, {\mathbf{X}}]$ with the corresponding prior expectations ${\mathbb{E}}[{\mathbf{H}}|{\mathbf{X}}]{\mathbb{E}}[{\mathbf{H}}^T|{\mathbf{X}}]$ and ${\mathbb{E}}[{\mathbf{H}}|{\mathbf{X}}]$. Hence, the ML estimator makes a more sufficient use of the available information, in the sense that it takes ${\mathbf{Y}}$ into account while BP does not.
	
	To relate $\frac{\partial {\mathcal{E}}}{\partial {\mathbf{W}}_0 }$ to $\frac{\partial {\mathcal{Q}}}{\partial {\mathbf{W}}_0 }$, we require the  lemma below.
	\begin{lem}\label{lemma}
		Let $\mathbf{U}$ be a matrix of random numbers with ${\mathbf{U}}(:,j)\sim {\mathcal{N}}(\mathbb{E}[\mathbf{H}(:,j)|\mathbf{X}(:,j)], \rho^2)$. If ${\mathbf{Y}}$ are generated according to $\mathbf{y}_j|\mathbf{U}(:,j)\!\sim\!\mathcal{N}(\mathbf{y}_j|{\mathbf{W}}_{\!1}\mathbf{U}(:,j) \!+\! \mathbf{b}_1,\sigma_1^2\mathbf{I})$, $\forall$ $\!j$, then the $\frac{\partial {\mathcal{E}}}{\partial {\mathbf{W}}_0 }$ in \eqref{grad_E_W0} can be equivalently expressed as
		\begin{align}
		\frac{\partial {\mathcal{E}}}{\partial {\mathbf{W}}_0 }  &\!=\! -\! \frac{1}{\rho^2}\Big[\!\left(\left(\sigma^2_1{\mathbf{I}} \!+\! \rho^2{\mathbf{W}}_1^T{\mathbf{W}}_1\! \right)\left({\mathbb{E}}[{\mathbf{H}}|{\mathbf{X}}] \!-\! {\mathbb{E}}[{\mathbf{U}}|{\mathbf{Y}}, {\mathbf{X}}]\right) \right)\nonumber \\
		&\hspace{2cm}\odot \mathrm{Var}(\mathbf{H}|\mathbf{X})/\sigma_0^2\Big] {\mathbf{X}}^T. \label{grad_E_W0_v2}
		\end{align}
	\end{lem}
	\begin{proof}
		Since the prior $p(\mathbf{U})$ and the conditional $p(\mathbf{Y}|\mathbf{U})$ are both Gaussian, the joint distribution $p(\mathbf{Y},\mathbf{U})$ is also Gaussian. As a result, the posterior $p(\mathbf{U}|\mathbf{Y})$ is a Gaussian distribution with the mean  given by ${\mathbb{E}}[{\mathbf{U}}|{\mathbf{Y}}, {\mathbf{X}}] = \left((\sigma^2_1/\rho^2){\mathbf{I}} + {\mathbf{W}}_1^T{\mathbf{W}}_1\right)^{-1}{\mathbf{W}}_1^T ({\mathbf{Y}}-{\mathbb{E}}[{\mathbf{Y}}|{\mathbf{X}}]) + {\mathbb{E}}[{\mathbf{H}}|{\mathbf{X}}]$. It then follows that ${\mathbf{W}}_1^T ({\mathbf{Y}}-{\mathbb{E}}[{\mathbf{Y}}|{\mathbf{X}}]) = \left((\sigma^2_1/\rho^2){\mathbf{I}} + {\mathbf{W}}_1^T{\mathbf{W}}_1\right)\left({\mathbb{E}}[{\mathbf{U}}|{\mathbf{Y}}, {\mathbf{X}}] - {\mathbb{E}}[{\mathbf{H}}|{\mathbf{X}}]\right)$, which is substituted into \eqref{grad_E_W0} to yield \eqref{grad_E_W0_v2}.
	\end{proof}
	As (\ref{grad_E_W0_v2}) holds for any $\rho^2>0$, it is also true when $\rho^2\approx0$, in which case the value of $\rho^2$ has little influence on the direction of $\frac{\partial {\mathcal{E}}}{\partial {\mathbf{W}}_0 }$; Therefore, we can make $\rho^2$ sufficiently small such that $\left(\sigma^2_1{\mathbf{I}}+\rho^2{\mathbf{W}}_1^T{\mathbf{W}}_1\right)\approx\sigma^2_1{\mathbf{I}}$, and consequently $\frac{\partial {\mathcal{E}}}{\partial {\mathbf{W}}_0 } \approx-\frac{\sigma^2_1}{\sigma_0^2\rho^2}\Big[\left(\left({\mathbb{E}}[{\mathbf{H}}|{\mathbf{X}}] \!-\! {\mathbb{E}}[{\mathbf{U}}|{\mathbf{Y}}, {\mathbf{X}}]\right) \right)\odot \mathrm{Var}(\mathbf{H}|\mathbf{X})\Big] {\mathbf{X}}^T$. Comparing the latter equation to \eqref{grad_W0_reg}, we see that the gradients $\frac{\partial\mathcal{E}}{\partial {\mathbf{W}}_0 }$ and $\frac{\partial\mathcal{Q}}{\partial {\mathbf{W}}_0}$ are different in three aspects: (\emph{i}) the ${\mathbb{E}}[{\mathbf{H}}|{\mathbf{Y}}, {\mathbf{X}}]$ in $\frac{\partial\mathcal{Q}}{\partial {\mathbf{W}}_0}$  is replaced by ${\mathbb{E}}[{\mathbf{U}}|{\mathbf{Y}}, {\mathbf{X}}]$ in $\frac{\partial\mathcal{E}}{\partial {\mathbf{W}}_0}$; (\emph{ii}) a new factor  $\mathrm{Var}(\mathbf{H}|\mathbf{X})$ arises in  $\frac{\partial\mathcal{E}}{\partial {\mathbf{W}}_0}$; and (\emph{iii}) the multiplicative constants are different. Since (\emph{iii}) has no influence on the directions of the gradients, we focus on (\emph{i}) and (\emph{ii}). The new factor  $\mathrm{Var}(\mathbf{H}|\mathbf{X})$ in  $\frac{\partial\mathcal{E}}{\partial {\mathbf{W}}_0}$ does not depend on $\mathbf{Y}$, so it plays no direct role in back-propagating the information from the output layer to the input layer. The only term of $\frac{\partial\mathcal{E}}{\partial {\mathbf{W}}_0}$ that contains $\mathbf{Y}$ is ${\mathbb{E}}[{\mathbf{U}}|{\mathbf{Y}}, {\mathbf{X}}]$, which plays the primary and direct role in sending back the information from the output layer when updating the input-to-hidden weights $\mathbf{W}_0$. Since ${\mathbb{E}}[{\mathbf{U}}|{\mathbf{Y}}, {\mathbf{X}}]$ is obtained under the assumption that ${\mathbf{Y}}$ is generated from Gaussian latent variables $\mathbf{U}$, it is clear that the gradient $\frac{\partial\mathcal{E}}{\partial {\mathbf{W}}_0}$ used by BP does not fully reflect the underlying truncated characteristics of ${\mathbf{H}}$ in the TGGM model. On the contrary, ${\mathbb{E}}[{\mathbf{H}}|{\mathbf{Y}}, {\mathbf{X}}]$ is the true posterior mean of $\mathbf{H}$ under the truncation assumption.
	
	In summary, BP uses update rules that are closely related to those of the ML estimator, but it does not fully exploit the available information in updating the TGGM parameters. In particular, BP ignores $\mathbf{Y}$ when it uses $\mathbb{E}(\mathbf{H}|\mathbf{X})$, instead of $\mathbb{E}(\mathbf{H}|\mathbf{Y},\mathbf{X})$, to update $\mathbf{W}_1$; it makes an incorrect assumption about the latent variables when it uses $\mathbb{E}(\mathbf{U}|\mathbf{Y},\mathbf{X})$, instead of $\mathbb{E}(\mathbf{H}|\mathbf{Y},\mathbf{X})$, to update $\mathbf{W}_0$. These somewhat defective update equations are attributed to the fact that BP makes a wrong assumption from the very beginning, i.e., BP assumes $p(\mathbf{y}|\mathbf{x})$ is a Gaussian distribution while the distribution is truly non-Gaussian as shown in (\ref{p_y_x_modul}). For these reasons, BP usually produces worse learning results for a TGGM than the ML estimator, and this will be discussed further in the experiments.
	
	\subsection{Technical Details}\label{sec:technical-details}
	A key step of the ML estimator is calculation of the prior and posterior expectations ${\mathbb{E}}[\cdot|{\mathbf{X}}]$ and ${\mathbb{E}}[\cdot|{\mathbf{Y}}, {\mathbf{X}}]$ in (\ref{grad_W0_reg}) and (\ref{grad_w1_reg}). Since $\mathbf{P}_0=\sigma_0^2\mathbf{I}_m$ is diagonal, the components in $\mathbf{h}_i|\mathbf{x}_i$ are independent; further, the training samples are assumed independent to each other. Therefore $p(\mathbf{H}|\mathbf{X})$ factorizes into a product of univariate truncated normal densities, $p(\mathbf{H}|\mathbf{X})=\prod_{i=1}^N\prod_{k=1}^m\mathcal{N}_T(\mathbf{h}_i(k)|\mathbf{W}_0(k,:)\mathbf{x}_i,\sigma_0^2)$, where each univariate density is associated with a single truncated variable and a particular training sample. Each of these densities has its mean and variance given by  ${\mathbb{E}}[{\mathbf{h}}_i(k)|{\mathbf{x}}_i]   =  g(\mathbf{W}_0(k,:)\mathbf{x}_i+\mathbf{b_0}(k),\sigma_0^2)$ and $\mathrm{Var}[{\mathbf{h}}_i(k)|{\mathbf{x}}_i]   = \omega^2\left(\mathbf{W}_0(k,:)\mathbf{x}_i+\mathbf{b_0}(k),\sigma_0^2\right)$, respectively, where $g(\cdot,\cdot)$ is defined in \eqref{truncated_mean_general} , and  
	\begin{eqnarray}\label{eq:trunc-var-1D}
	\omega^2\left(\mu,\sigma\right) &\triangleq&\sigma^2\left(1-\frac{\mu}{\sigma}\frac{\phi\left(\frac{\mu}{\sigma}\right)}{\Phi\left(\frac{\mu}{\sigma}\right)} - \frac{\phi^2\left(\frac{\mu}{\sigma}\right)}{\Phi^2\left(\frac{\mu}{\sigma}\right)}\right)
	\end{eqnarray}
	is the variance of the truncated normal ${\mathcal{N}}_T(z|\mu, \sigma^2)$ \cite{johnson1994continuous}.  Due to the independences, one can easily compute ${\mathbb{E}}\left[{\mathbf{H}}{\mathbf{H}}^T|{\mathbf{X}}\right] = \sum_{i=1}^N{\mathbb{E}}\left[{\mathbf{h}}_i{\mathbf{h}}_i^T|{\mathbf{x}}_i\right]$, with ${\mathbb{E}}\left[{\mathbf{h}}_i{\mathbf{h}}_i^T|{\mathbf{x}}_i\right] = {\mathbb{E}}[{\mathbf{h}}_i|{\mathbf{x}}_i]{\mathbb{E}}[{\mathbf{h}}_i^T|{\mathbf{x}}_i] + \text{diag}\left(\mathrm{Var}[{\mathbf{h}}_i|{\mathbf{x}}_i]\right)$. 
	
	For the posterior expectation ${\mathbb{E}}[\cdot|{\mathbf{Y}}, {\mathbf{X}}]$, it could be computed by means of numerical integration. Multivariate integrations in normal distributions have been well studied and many effective algorithms have been developed \cite{genz1992numerical, genz2009computation}. 
	Another approach is to use the mean-field variational Bayesian (VB) method \cite{jordan1999introduction,Biship01variational}, which approximates the true posterior $p({\mathbf{H}}|{\mathbf{Y}}, {\mathbf{X}})$ with a factorized distribution $q({\mathbf{H}}|\widetilde{\boldsymbol{\Theta}}) = \prod_{i=1}^N\prod_{k=1}^m q({\mathbf{h}}_i(k)|\widetilde{\boldsymbol{\Theta}})$, parameterized by $\widetilde{\boldsymbol{\Theta}}$. The approximate posterior is found by minimizing $\mathrm{KL}\left(q(\mathbf{H}|\widetilde{\boldsymbol{\Theta}})||p(\mathbf{H}|{\mathbf{Y}}, {\mathbf{X}}; {\boldsymbol{\Theta}})\right)$, or maximizing the lower bound $\mathcal{Q}_{q(\cdot|\widetilde{\boldsymbol{\Theta}})}(\boldsymbol{\Theta})$, as shown in (\ref{eq:lower-bound}). 
	
	As $\mathbf{h}_i$ is independent of $\mathbf{h}_j$, $\forall$ $i,j$, given $\mathbf{Y}$ and $\mathbf{X}$,  the KL distance can be equivalently expressed as $\sum_{i=1}^N\mathrm{KL}\left(q(\mathbf{h}_i|\widetilde{\boldsymbol{\Theta}})||p(\mathbf{h}_i|{\mathbf{Y}}, {\mathbf{X}}; {\boldsymbol{\Theta}})\right)$ and each term in the sum can be minimized independently. Given $\{q({\mathbf{h}}_j(\ell)):\ell\neq{}k\}$, the $i$th term of the KL distance is minimized by
	\begin{equation} \label{q_hki_general_form}
	q({\mathbf{h}}_i(k)|\widetilde{\boldsymbol{\Theta}}) \propto e^{\left\langle \ln p({\mathbf{y}}_i, {\mathbf{h}}_i| {\mathbf{x}}_i) \right\rangle_{-k}},
	\end{equation}
	where $p({\mathbf{y}}_i, {\mathbf{h}}_i| {\mathbf{x}}_i) = {\mathcal{N}}_T\left({\mathbf{h}}_i \left|  {\mathbf{W}}_0{\mathbf{x}}_i+{\mathbf{b}}_0, \sigma_0^2 {\mathbf{I}}_m \right.\right) \times{\mathcal{N}}\left({\mathbf{y}}_i \left| {\mathbf{W}}_1{\mathbf{h}}_i+{\mathbf{b}}_1, \sigma^2_1{\mathbf{I}}_n\right.\right)$ and $\left\langle\cdot \right\rangle_{-k}$ denotes the expectation w.r.t. $\prod_{\ell\ne k} q({\mathbf{h}}_j(\ell))$. From \eqref{q_hki_general_form}, one obtains 
	\begin{equation}
	q\left( {\mathbf{h}}_i(k) |\widetilde{\boldsymbol{\Theta}}\right) \!=\! {\mathcal{N}}_T\left({\mathbf{h}}_i(k) \left| {\boldsymbol{\xi}}_i(k), \frac{1}{{\mathbf{P}}(k,k)} \right.\right), \label{q_update}
	\end{equation}
	where ${\mathbf{P}} \triangleq \frac{1}{\sigma^2_0}{\mathbf{I}}_m + \frac{1}{\sigma^2_1}{\mathbf{W}}_1^T{\mathbf{W}}_1$, ${\boldsymbol{\xi}}_i$ is a vector with its $k$-th element defined as $
	{\boldsymbol{\xi}}_i(k) = \frac{{\boldsymbol{\gamma}}_i(k) - {\mathbf{\tilde P}}(k,-k)\left\langle{\mathbf{h}}_i(-k)\right\rangle_{-k}}{{\mathbf{P}}(k,k)}$, ${\boldsymbol{\gamma}}_i \triangleq \frac{1}{\sigma^2_0}\left({\mathbf{W}}_0{\mathbf{x}}_i \!+\! {\mathbf{b}}_0\right) + \frac{1}{\sigma^2_1}{\mathbf{W}}_1^T({\mathbf{y}}_i \!-\! {\mathbf{b}}_1)$, ${\mathbf{\tilde P}} \triangleq {\mathbf{P}} - \text{diag}({\mathbf{P}})$, ${\mathbf{P}}(k,-k)$ is the $k$-th row of ${\mathbf{P}}$ with its $k$-th element deleted, and  ${\mathbf{h}}_i(-k)$ is the subvector of ${\mathbf{h}}_i$ missing the $k$-th element.
	
	The KL distance $\mathrm{KL}\left(q(\mathbf{h}_i|\widetilde{\boldsymbol{\Theta}})||p(\mathbf{h}_i|{\mathbf{Y}}, {\mathbf{X}}; {\boldsymbol{\Theta}})\right)$ monotonically decreases as one cyclically computes (\ref{q_update}) through $k=1,2,\cdots,m$. One shall perform enough cycles until the KL distance converges. Upon convergence, $q({\mathbf{h}}_i(k)|\widetilde{\boldsymbol{\Theta}})$ is used as the best approximation to $p(\mathbf{h}_i|{\mathbf{Y}}, {\mathbf{X}}; {\boldsymbol{\Theta}})$, $\forall$ $i,k$, and their means and variances, as given by the formulae in (\ref{truncated_mean_general}) and (\ref{eq:trunc-var-1D}), are used to compute the posterior expectations ${\mathbb{E}}[\cdot|{\mathbf{Y}}, {\mathbf{X}}]$ in (\ref{grad_W0_reg}) and (\ref{grad_w1_reg}).
	
	After (\ref{grad_W0_reg})-(\ref{grad_w1_reg}) are computed and the TGGM parameters in $\boldsymbol{\Theta}$ are improved based on the gradient ascent in (\ref{eq:gradient-EM}) , one iteration of the ML estimator is completed. Given the updated $\boldsymbol{\Theta}$, one then repeat the cycles with (\ref{q_update}) to find the approximate posteriors and again make another update of $\boldsymbol{\Theta}$, and so on. The complete ML estimation algorithm is summarized in Algorithm 1, where $T_1$ represents the number of cycles with (\ref{q_update}) to find the best posterior distribution $q({\mathbf{H}}|\widetilde{\boldsymbol{\Theta}})$ for each newly-updated ${\boldsymbol{\Theta}}$. We can see that the complexity mainly comes from the estimation of expectation ${\mathbb{E}}[{\mathbf{h}}_i|{\mathbf{y}}_i, {\mathbf{x}}_i]$, which is  ${\mathcal{O}}(T_1M^2)$.
	
	\begin{algorithm}
		\begin{algorithmic}[1] 
			\caption{ML Estimator for TGGM Regression}
			\label{alg:alg_train_regression}
			\STATE Randomly initialize the model parameters ${\boldsymbol{\Theta}}$;
			\REPEAT
			\FOR{$t=1$ {\textbf{to}} $T_1$}
			\FOR{$k=1$ \textbf{to} $M$}
			\STATE Update ${\mathbb{E}}[{\mathbf{h}}_i(k)|{\mathbf{y}}_i, {\mathbf{x}}_i]$ using (\ref{truncated_mean_general});
			\STATE Replace the $k$-th value of ${\mathbb{E}}[{\mathbf{h}}_i|{\mathbf{y}}_i, {\mathbf{x}}_i]$ with ${\mathbb{E}}[{\mathbf{h}}_i(k)|{\mathbf{y}}_i, {\mathbf{x}}_i]$;
			\ENDFOR
			\ENDFOR
			\STATE Compute ${\mathbb{E}}[{\mathbf{h}}_i{\mathbf{h}}_i^T|{\mathbf{y}}_i, {\mathbf{x}}_i]$ using (\ref{truncated_mean_general}) and (\ref{eq:trunc-var-1D});
			\STATE Calculate the gradients of log-likelihood using \eqref{grad_W0_reg} and \eqref{grad_w1_reg};
			\STATE Update model parameters ${\boldsymbol{\Theta}}$ with gradient ascend;
			\UNTIL{Convergence of log-likelihood}
		\end{algorithmic}
	\end{algorithm}
	
	Finally, it should be noted that the expectations require frequent calculation of the ratio of $\frac{\phi(a)}{\Phi(a)}$. In practice, if it is computed directly, we easily encounter two issues. First, repeated computation of the integration $\Phi(a) = \int_{-\infty}^{a}{\frac{1}{\sqrt{2\pi}} e^{-\frac{z^2}{2}} dz}$ is a waste of time; second, when $a$ is small, e.g. $a \le-37$,  the CDF $\Phi(a)$ and the PDF $\phi(a)$ are so tiny that a double-precision floating number can no longer represent them accurately. If we compute them with the double-precision numbers, we easily encounter the issue of $\frac{0}{0}$. Fortunately, both these issues  can be solved by using a lookup  table. Specifically, we pre-compute $\frac{\phi(a)}{\Phi(a)}$ at densely-sampled discrete values of $a$ using high-accuracy computation, such as the symbolic calculation in Matlab, and store the results in a table. When we need  $\frac{\phi(b)}{\Phi(b)}$ for any $b$, we look for two values of $a$ that are closest to $b$ and use the interpolation of the two $\frac{\phi(a)}{\Phi(a)}$ to estimate $\frac{\phi(b)}{\Phi(b)}$.
	
	
	\section{Extension to Other Learning Tasks}

	\subsection{Nonlinear Classification}\label{sec:classification}
	Let $c\in\{1,\cdots,n\}$ denote $n$ possible classes. Let $\mathbf{T}_{c}\in\mathbb{R}^{(n-1)\times{}n}$ be a class-dependent matrix obtained from $-\mathbf{I}_n$ by setting the $c$-th column to one and deleting the $c$-th row \cite{liao2007quadratically}. We define a nonlinear classifier as 
	\begin{align}
	\hspace{-0.2cm}p(c) &= \int_{0}^{\infty}\int_{0}^{\infty}{\mathcal{N}}\left(\mathbf{z}\left|\mathbf{T}_c({\mathbf{W}}_1\mathbf{h}+{\mathbf{b}}_1), \mathbf{T}_c\mathbf{T}_c^T\right.\right)d\mathbf{z}
	\cr&\hspace{1cm}\times{\mathcal{N}}_T\left(\mathbf{h}\left|  \mathbf{W}_0\mathbf{x}+\mathbf{b}_0, \sigma_0^2 {\mathbf{I}}_m \right.\right)d\mathbf{h}. \label{class_hi}
	\end{align}
	The inner integral is due to the multinomial probit model \cite{albert1993bayesian} which transforms the TGGM's output vector $\mathbf{y}$ in (\ref{joint_pdf_y_h_x}) into a class label according to  $c=\mathrm{arg}\max_k\mathbf{y}(k)=\mathrm{arg}\max_k\mathbb{I}(\mathbf{T}_{k}\mathbf{y}\geq\mathbf{0})$. Therefore, $p(c)=p(\mathbf{T}_{c}\mathbf{y}\geq\mathbf{0})=\int_{\mathbf{T}_{c}\mathbf{y}\geq\mathbf{0}}{\mathcal{N}}\left(\mathbf{y}\left|{\mathbf{W}}_1\mathbf{h}+{\mathbf{b}}_1, \mathbf{I}_n\right.\right)d\mathbf{y}$. A change of variables $\mathbf{z}\triangleq\mathbf{T}_{c}\mathbf{y}$ leads to $p(c)=\int_{0}^{\infty}{\mathcal{N}}\left(\mathbf{z}\left|\mathbf{T}_c({\mathbf{W}}_1\mathbf{h}+{\mathbf{b}}_1), \mathbf{T}_c\mathbf{T}_c^T\right.\right)d\mathbf{z}$. 
	
	The model described by \eqref{class_hi} can be trained by an ML estimator, similarly to the case of regression, with the main difference being the additional latent vector $\mathbf{z}$, which can be treated in a similar way as $\mathbf{h}$. The posterior $p(\mathbf{z},\mathbf{h}|\mathbf{x},c)$ is still a truncated Gaussian distribution, whose moments can be computed using the methods in Section \ref{reg_training}.  
	The model predicts the class label of $\mathbf{x}$ using the rule $\hat c = \arg \max \limits_{k} {\mathbb{E}}[{\mathbf{y}}(k)|{\mathbf{x}}]$, where ${\mathbb{E}}[{\mathbf{y}}|{\mathbf{x}}]={\mathbf{W}}_1{\mathbb{E}}[{\mathbf{h}}|{\mathbf{x}}] + {\mathbf{b}}_1$ and ${\mathbb{E}}[{\mathbf{h}}(k)|{\mathbf{x}}_i] = g\left({\mathbf{W}}_0(k,:){\mathbf{x}} + {\mathbf{b}}_0(k), \sigma_0\right)$.
	
	\subsection{Deep Learning}\label{sec:deep-learning}
	The TGGM defined in (\ref{joint_pdf_y_h_x}) can be viewed as a neural network, where the input, hidden, and output layers are respectively constituted by $\mathbf{y}$, $\mathbf{h}$, and $\mathbf{x}$, and the hidden layer has outgoing connections to the output layer and incoming connections from the input layer. The topology of the hidden layer is determined by $\mathbf{P}_0$. So far, we have focused on $\mathbf{P}_0=\sigma_0^2\mathbf{I}_m$ , which defines a single layer of hidden nodes that are not interconnected. By using a more sophisticated $\mathbf{P}_0$, we can construct a deep TGGM with two or more hidden layers and enhanced representational versatility. 
	
	
	As an example, we let ${\mathbf{h}}\!\triangleq\! [{\mathbf{h}}^{(1)}; {\mathbf{h}}^{(2)}]$ and define 
	$p({\mathbf{h}}|\mathbf{x}) 
	\propto\exp\{ -\frac{1}{2\sigma^{(1)2}_0}\|{\mathbf{h}}^{(1)} - {\mathbf{W}}_0^{(1)}{\mathbf{x}}-{\mathbf{b}}_0^{(1)}\|^2 \} 
	\times \exp\{ -\frac{1}{2\sigma^{(2)2}_0} \|{\mathbf{h}}^{(2)} \! -{\mathbf{W}}_0^{(2)}{\mathbf{h}}^{(1)} - {\mathbf{b}}_0^{(2)} \|^2 \} \mathbb{I}\left({\mathbf{h}} \ge {\mathbf{0}} \right)$.
	Taking into account the normalization, the distribution can be written as $p({\mathbf{h}}) = {\mathcal{N}}_T\left({\mathbf{h}}\left|{\boldsymbol{\zeta}}, \mathbf{P}_0^{-1}\right.\right)$, where $\boldsymbol{\zeta}$ and $\mathbf{P}_0$ depend on $\{\mathbf{W}_0^{(t)},\mathbf{b}_0^{(t)},\sigma_0^{(t)2}\}_{t=1}^2$. This distribution, along with $p({\mathbf{y}}|\mathbf{h}^{(2)}) = {\mathcal{N}}({\mathbf{y}} | {\mathbf{W}}_1{\mathbf{h}}^{(2)} +{\mathbf{b}}_1, \sigma^2_1{\mathbf{I}}_n)$, yields a TGGM with two hidden layers. 
	Extensions to three or more hidden layers and to classification can be constructed similarly. A deep TGGM can be learned by using EM to maximize the likelihood, wherein the derivatives of the lower bound $\mathcal{Q}$ can be represented as $\frac{\partial {\mathcal{Q}} }{\partial {\boldsymbol{\Theta}}} = -{\mathbb{E}}\left[\frac{\partial E}{\partial {\boldsymbol{\Theta}}}| {\mathbf{Y}}, {\mathbf{X}} \right] + {\mathbb{E}}\left[\frac{\partial E}{\partial {\boldsymbol{\Theta}}}| {\mathbf{X}} \right]$, as in Section \ref{sec:reg-ML}. 
	
	Training a multi-hidden-layer TGGM is almost the same as training a single-hidden-layer TGGM, except for the difference in estimating the prior expectation ${\mathbb{E}}[\frac{\partial E}{\partial {\boldsymbol{\Theta}}}|{\mathbf{X}}]$. In the single-hidden-layer case, since $\mathbf{P}_0$ is diagonal, $p({\mathbf{h}}_i| {\mathbf{x}}_i; {\boldsymbol{\Theta}})$ factorizes into a set of univariate truncated normals, and therefore the expectation can be computed efficiently. In the multi-hidden-layer case, however, $p({\mathbf{h}}_i| {\mathbf{x}}_i; {\boldsymbol{\Theta}})$ is a multivariate truncated normal, and thus the prior expectation is as difficult to compute as the posterior expectation. Following Section \ref{sec:technical-details}, we use mean-filed VB to approximate 
	$p({\mathbf{h}}_i|{\mathbf{x}}_i, {\boldsymbol{\Theta}})$ by factorized univariate truncated normals and estimate ${\mathbb{E}}[\frac{\partial E}{\partial {\boldsymbol{\Theta}}}|{\mathbf{X}}]$ with the univariate distributions. 
	
	In practice, we find that starting from the VB approximation of $p({\mathbf{h}}_i|\mathbf{y}_i,{\mathbf{x}}_i, {\boldsymbol{\Theta}})$ can improve the VB approximation of $p({\mathbf{h}}_i|{\mathbf{x}}_i, {\boldsymbol{\Theta}})$,  a feature similar to that observed in the contrastive divergence (CD) \cite{hinton2002training}. 
	
	\begin{table*} {\tiny } 
		\centering
		\caption{Averaged Test RMSE and Std. Errors}
		\vspace{2.25mm}
		\begin{tabular}{l c c c c c c}
			Dataset & N & d & ReLU-BP & ReLU-PBP  & TGGM-BP  & TGGM-ML \\ \hline
			Boston Housing & 506 & 13 & 3.228$\pm$0.1951 &  3.014$\pm$ 0.1800 &   2.927 $\pm$ 0.2910 & \textbf{2.820$\pm$ 0.2565} \\
			Concrete Strength & 1030 & 8 & 5.977$\pm$ 0.0933 & 5.667$\pm$ 0.0933& 5.657 $\pm$ 0.2685 & \textbf{5.395$\pm$ 0.2404} \\
			Energy Efficiency & 768 & 8 & 1.098 $\pm$ 0.0738 & 1.804 $\pm$ 0.0481 &  \textbf{1.029 $\pm$ 0.1206}  &1.244 $\pm$ 0.0979 \\
			Kin8nm & 8192 & 8 & 0.091$\pm$ 0.0015 & 0.098$\pm$ 0.0007 &  0.088 $\pm$ 0.0025  &\textbf{0.083 $\pm$ 0.0034}    \\
			Naval Propulsion & 11934 & 16 & 0.001$\pm$ 0.0001 & 0.006$\pm$ 0.0000 &  \textbf{0.00057$\pm$ 0.0001}  &0.003 $\pm$ 0.0002 \\
			Cycle Power Plant & 9568 & 4 & {4.182$\pm$ 0.0402} & 4.124$\pm$ 0.0345 &  \textbf{3.949 $\pm$ 0.1478}  &4.183 $\pm$ 0.0955 \\
			Protein Structure & 45730 & 9 & 4.539$\pm$ 0.0288 & 4.732$\pm$ 0.0130 &  4.477$\pm$ 0.0483   &\textbf{4.431 $\pm$ 0.0292} \\
			Wine Quality Red & 1599 & 11 & 0.645$\pm$ 0.0098 & 0.635$\pm$0.0079&   0.640 $\pm$ 0.0469  &\textbf{0.625 $\pm$ 0.0340}\\
			Yacht Hydrodynamic & 308 & 6 & 1.182$\pm$ 0.1645 & 1.015$\pm$ 0.0542 & 0.957 $\pm$ 0.2319  &\textbf{0.841 $\pm$ 0.2028} \\
			Year Prediction MSD & 515,345 & 90 & 8.932 $\pm$ N/A & \textbf{8.878 $\pm$ N/A}    & 8.918 $\pm$ N/A  & 9.002 $\pm$ N/A \\ \hline
		\end{tabular}\label{regression_result}
	\end{table*}
	
	\section{Experiments}\label{sec:experiments}
	We report the performance of the proposed TGGM models on publicly available data sets, in comparison to competing models. In all experiments below, RMSProp \cite{tieleman2012lecture} is applied to update the model parameters by using the current estimated gradients, with RMSprop delay set to be $0.95$.

	\textbf{Regression} The root mean square error (RMSE), averaged over multiple trials of splitting each data set into training and testing subsets, is used as a performance measure to evaluate the TGGM against the ReLU neural network. The  results reported in \cite{hernandez2015probabilistic} are used as the reference to the performances of ReLU neural networks. The comparison is based on the same data and same training/testing protocols in \cite{hernandez2015probabilistic}, by using a consistent setting for the TGGM as follows: a single hidden layer is used in the TGGM for all data sets, with $100$ hidden nodes used for Protein Structure and Year Prediction MSD, the two largest data sets, and $50$ hidden nodes used for the other data sets. 
	
	Two methods, BP and ML estimation, are applied to train each TGGM, resulting in two versions of the TGGM for each data set, referred to TGGM-BP and TGGM-ML, respectively. For both training methods, ${\boldsymbol{\Theta}}$ is initialized as Gaussian random numbers, with each component a random draw from ${\mathcal{N}}(0, 0.01)$.  To speed up, each gradient update uses a mini-batch of training samples, resulting in stochastic gradient search. The batch size is $100$ for the two largest data sets and $50$ for the others. For ML estimation, the number of cycles used by mean-field VB is set to $10$, and $\sigma_1^2=\sigma_0^2=0.5$. 
	
	
	The testing RMSE's of the TGGM are summarized in Table \ref{regression_result}, alongside the corresponding results  \cite{hernandez2015probabilistic} for the ReLU neural networks trained by BP and probabilistic backpropagation (PBP). BP provides a point estimate of the model parameters while PBP provides the posterior distribution. It is seen from Table \ref{regression_result} that TGGM-BP performs slightly better than ReLU on most data sets. The gain may be attributed to the soft activation function $g(\mu, \sigma)$ which provides the freedom in choosing appropriate $\sigma_0$ according to data's characteristics, unlike ReLU which fixes $\sigma$ to $0$. The nonzero slope in $g(\mu,\sigma)$ for $\mu<0$ may be another contributing factor, as it has been shown in \cite{he2015delving} that replacing the zero part of ReLU with a sloping line leads to better results. 
	Furthermore, it can be observed that TGGM-ML outperforms TGGM-BP on most data sets. This is because ML-based training fully exploits the flexibilities provided by a probabilistic model and, as discussed in Section \ref{sec:reg-ML-vs-BP}, is more accurate in reflecting the underlying model, in contrast to TGGM-BP which made a wrong assumption about the model at the very beginning.
	
	
	We also observe that, if $\sigma_0^2$ is set close to $0$, the performance of TGGM-BP approaches that of the ReLU neural network. This is not surprising since  $g(\mu, \sigma_0)$ approaches the ReLU activation function as $\sigma^2_0\to 0$. As we increase the value of $\sigma_0^2$, the TGGM's performance improves gradually, until it reaches a saturating value. This is reasonable because $g(\mu, \sigma_0)$ becomes more linearly (w.r.t. $\mu$) as $\sigma^2_0$ becomes larger, which weakens its nonlinear representational abilities. Empirically, we find that TGGM-BP performs similarly within an appropriate range of $\sigma^2_0$. The results in Table \ref{regression_result} are based on $\sigma_0^2 = 0.01$, which is found to be a good setting for all data sets. The impact of $\sigma^2_0$ on the RMSE results is illustrated in Fig. \ref{TGGM-BP_sigma}. Note that $\sigma^2_0$ can also be learned directly from the data, which is an interesting future work.
	
	We found that the optimal $\sigma_0^2$ for TGGM-ML is typically larger than that for TGGM-BP. This is perhaps because a larger $\sigma_0^2$ provides increased flexibility in the TGGM, which can be exploited by a probabilistic inference method like expectation-maximization. As a result, we use $\sigma_0^2=0.5$ for TGGM-ML in the experiments.
	
	\begin{figure}
		\centering
		\includegraphics[width=0.30\textwidth]{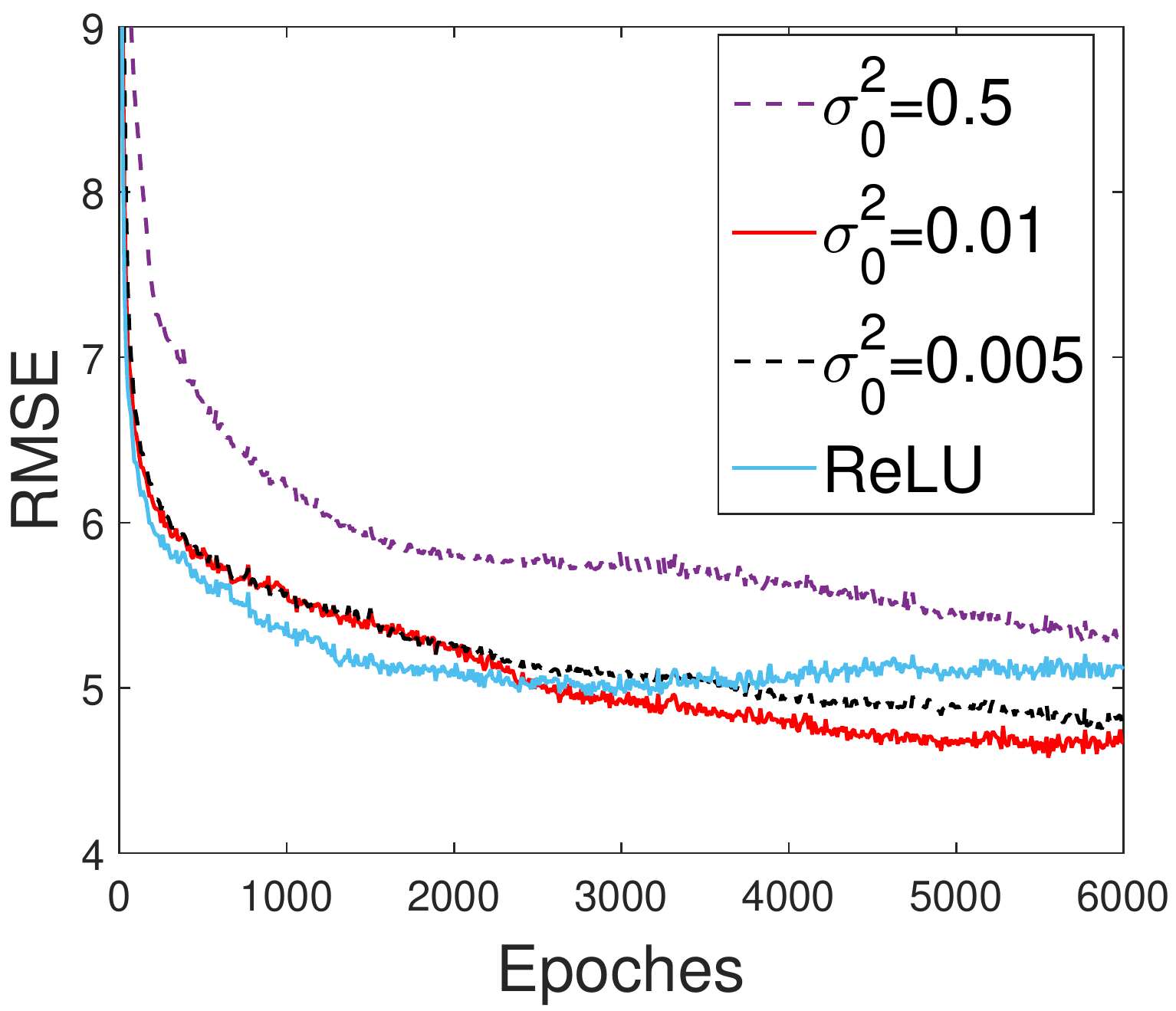}\\
		\caption{Illustration of the impact of $\sigma_0^2$ on TGGM-BP, based on the RMSEs for Concrete Strength in a single trial.}
		\label{TGGM-BP_sigma}
		\vspace{-2mm}
	\end{figure}

	\textbf{Classification}  Three public benchmark data sets are considered for this task: MNIST, 20 NewsGroups, and Blog. The MNIST data set includes 60,000 training images and 10,000 testing images of handwritten digits of zero through ten. The 20 NewsGroups data sets is composed of 18,845 documents, written with a vocabulary of 2,000 words, from 20 different groups, with the data partitioned into a training set of 11,315 documents  and a testing set of 7,531 documents \cite{li2016shape}. The Blog data set contains 13,245 documents, written with 17,292 words, about the US presidential elections; the data are partitioned into 7,832 training documents and 5,413 testing documents \cite{chen2015differential}. One-hidden-layer and two-hidden-layer TGGM models are considered, with each hidden layer containing 100 or 200 nodes. Similar to the regression model, the TGGM classifier is trained by both BP and ML, with the resulting models termed as TGGM-BP and TGGM-ML, respectively.
	
	The models are randomly initialized with Gaussian random numbers drawn from ${\mathcal{N}}(0, 0.01)$. The step-size for gradient ascent is chosen from $[10^{-4},  5\times10^{-3}]$ by maximizing the accuracy on a cross-validation set. The TGGMs use a minibatch size of 500 for MNIST and 200 for the other two data sets, while the ReLU uses 100 for all data sets. Variance parameters $\{\sigma^2_0, \sigma^2_1\}$ are set to $0.5$ for TGGM-ML and $0.01$ for TGGM-BP,  in both one and two-layer models. When ML estimation is applied, the number of VB cycles is initially set to 30 and then gradually decreases to 5. The data sets are also used to train and test a ReLU neural network implemented in Caffe \cite{jia2014caffe}, to produce the competing results for comparison. From Table \ref{classfication_result}, it can be seen that TGGM-BP generally outperforms ReLU on the two document data sets and maintains a comparable performance on the image data set, for both one- and two- hidden-layer models. It is  further observed that TGGM-ML has the best performance on all three data sets, with the best performance achieved by the two-hidden-layer models on MNIST and Blog, and by the one-hidden-layer model on 20 NewsGroup. 

	\begin{table} [t!]
		\centering
		\caption{Test Accuracy of Classification}
		\vspace{2.5mm}
		\begin{tabular}{l c  c c}
			Methods & MNIST  &  20 News & Blog\\ \hline
			ReLU (100) &  97.58\%  &   72.8\%  & 65.86\%  \\
			ReLU (200) &  97.89\% &  73.27\% & 67.02\%  \\
			ReLU (100-100) &  97.83\% & 69.94\% & 67.93\% \\
			ReLU (200-200) & 98.04\% & 69.91\% & 65.07\% \\ \hline
			TGGM-BP (100) &  97.52\%  &   73.65\%  & 67.50\% \\
			TGGM-BP (200) &  97.56\%  &  73.62\%     &67.52\%  \\
			TGGM-BP (100-100) & 97.76\% & 71.06\% & 66.82\% \\
			TGGM-BP (200-200)  & 98.12\% & 71.18\% & 67.73\% \\ \hline
			TGGM-ML (100)     & 97.75\% &  \textbf{73.74\%} & 69.83\% \\
			TGGM-ML (200)     & 97.97\%  & 73.38\% & 69.75\% \\
			TGGM-ML (100-100) & 98.05\% & 68.01\% & \textbf{69.89}\%\\
			TGGM-ML(200-200) & \textbf{98.31}\% & 67.52\%  & 66.64\% \\   \hline
		\end{tabular} \label{classfication_result}
	\end{table}


	\section{Conclusions}
	We have proposed a nonlinear statistical learning framework termed TGGM. By introducing truncated latent variables into the traditional GGM, we obtain the TGGM as a non-Gaussian nonlinear model with significantly enhanced modeling ability compared to the GGM. We demonstrate that regression and classification can be realized through appropriately constructed TGGMs. With carefully designed graphical structures, deep versions of TGGMs have also been obtained. It is shown that, for regression and classification, TGGMs can be approximately viewed as a deterministic neural network with an activation function similar to ReLU. Because of this, TGGMs can be trained with BP. However, BP does not exactly maximize the likelihood of a TGGM, due to the inherent Gaussian assumption it makes. To overcome this limitation, we have developed an algorithm to correctly maximize the likelihood under the truncated Gaussian assumption. Experimental results show that the TGGM trained by BP generally performs better than the ReLU network, indicating the advantage of the new activation function. It is further shown that the TGGM trained by ML learning achieves the best performance on most data sets in consideration. It should be emphasized that the tasks considered in this paper are only specific applications of the TGGM framework under special forms of the precision matrices. In the future, we will consider TGGMs with lateral connections between hidden nodes. We may also generalize the supervised TGGM to the unsupervised case, using constructs similar to RBMs. Moreover, investigation of how the quality of uncertainty estimates affects the performance is also of interest. 
	
	\clearpage
		\section*{Acknowledgements}
		The authors would like to thank the anonymous reviewers for their valuable and constructive comments. This research was supported in part by ARO, DARPA, DOE, NGA and ONR.
	
	\bibliography{reference} 

\begin{thebibliography}{35}
\providecommand{\natexlab}[1]{#1}
\providecommand{\url}[1]{\texttt{#1}}
\expandafter\ifx\csname urlstyle\endcsname\relax
  \providecommand{\doi}[1]{doi: #1}\else
  \providecommand{\doi}{doi: \begingroup \urlstyle{rm}\Url}\fi

\bibitem[Albert \& Chib(1993)Albert and Chib]{albert1993bayesian}
Albert, James~H and Chib, Siddhartha.
\newblock Bayesian analysis of binary and polychotomous response data.
\newblock \emph{Journal of the American statistical Association}, 88\penalty0
  (422):\penalty0 669--679, 1993.

\bibitem[Chen et~al.(2015)Chen, Buntine, Ding, Xie, and
  Du]{chen2015differential}
Chen, Changyou, Buntine, Wray, Ding, Ni, Xie, Lihua, and Du, Liang.
\newblock Differential topic models.
\newblock \emph{Pattern Analysis and Machine Intelligence, IEEE Transactions
  on}, 37\penalty0 (2):\penalty0 230--242, 2015.

\bibitem[Corduneanu \& Bishop(2001)Corduneanu and Bishop]{Biship01variational}
Corduneanu, A. and Bishop, C.
\newblock Variational {Bayesian} model selection for mixture distributions.
\newblock In \emph{AI and Statistics}, pp.\  27--34, 2001.

\bibitem[Dempster et~al.(1977)Dempster, Laird, and Rubin]{EM77Dempster}
Dempster, A., Laird, N., and Rubin, D.
\newblock Maximum likelihood from incomplete data via the {EM} algorithm.
\newblock \emph{Journal of Royal Statistical Society B}, 39:\penalty0 1--38,
  1977.

\bibitem[Downs et~al.(1999)Downs, MacKay, Lee, et~al.]{downs1999nonnegative}
Downs, Oliver~B, MacKay, David~JC, Lee, Daniel~D, et~al.
\newblock The nonnegative boltzmann machine.
\newblock In \emph{NIPS}, pp.\  428--434, 1999.

\bibitem[Frey(1997)]{frey1997continuous}
Frey, Brendan~J.
\newblock Continuous sigmoidal belief networks trained using slice sampling.
\newblock \emph{Advances in Neural Information Processing Systems}, pp.\
  452--458, 1997.

\bibitem[Frey \& Hinton(1999)Frey and Hinton]{frey1999variational}
Frey, Brendan~J and Hinton, Geoffrey~E.
\newblock Variational learning in nonlinear gaussian belief networks.
\newblock \emph{Neural Computation}, 11\penalty0 (1):\penalty0 193--213, 1999.

\bibitem[Galbraith et~al.(2002)Galbraith, Moustaki, Bartholomew, and
  Steele]{galbraith2002analysis}
Galbraith, JI, Moustaki, Irini, Bartholomew, David~J, and Steele, Fiona.
\newblock \emph{The analysis and interpretation of multivariate data for social
  scientists}.
\newblock CRC Press, 2002.

\bibitem[Gan et~al.(2015)Gan, Henao, Carlson, and Carin]{gan2015learning}
Gan, Zhe, Henao, Ricardo, Carlson, David~E, and Carin, Lawrence.
\newblock Learning deep sigmoid belief networks with data augmentation.
\newblock In \emph{AISTATS}, 2015.

\bibitem[Genz(1992)]{genz1992numerical}
Genz, Alan.
\newblock Numerical computation of multivariate normal probabilities.
\newblock \emph{Journal of computational and graphical statistics}, 1\penalty0
  (2):\penalty0 141--149, 1992.

\bibitem[Genz \& Bretz(2009)Genz and Bretz]{genz2009computation}
Genz, Alan and Bretz, Frank.
\newblock \emph{Computation of multivariate normal and t probabilities}, volume
  195.
\newblock Springer Science \& Business Media, 2009.

\bibitem[Glorot et~al.(2011)Glorot, Bordes, and Bengio]{glorot2011deep}
Glorot, Xavier, Bordes, Antoine, and Bengio, Yoshua.
\newblock Deep sparse rectifier neural networks.
\newblock In \emph{International Conference on Artificial Intelligence and
  Statistics}, pp.\  315--323, 2011.

\bibitem[He et~al.(2015)He, Zhang, Ren, and Sun]{he2015delving}
He, Kaiming, Zhang, Xiangyu, Ren, Shaoqing, and Sun, Jian.
\newblock Delving deep into rectifiers: Surpassing human-level performance on
  imagenet classification.
\newblock In \emph{Proceedings of the IEEE International Conference on Computer
  Vision}, pp.\  1026--1034, 2015.

\bibitem[Hern{\'a}ndez-Lobato \& Adams(2015)Hern{\'a}ndez-Lobato and
  Adams]{hernandez2015probabilistic}
Hern{\'a}ndez-Lobato, Jos{\'e}~Miguel and Adams, Ryan~P.
\newblock Probabilistic backpropagation for scalable learning of bayesian
  neural networks.
\newblock \emph{Proceedings of The 32nd International Conference on Machine
  Learning}, 2015.

\bibitem[Hinton \& Ghahramani(1997)Hinton and
  Ghahramani]{Hinton1997rectifierGaussian}
Hinton, G.~E. and Ghahramani, Z.
\newblock Generative models for discovering sparse distributed representations.
\newblock \emph{Phil. Trans. Roy. Soc., B}, 352:\penalty0 1177--90, 1997.

\bibitem[Hinton(2002)]{hinton2002training}
Hinton, Geoffrey~E.
\newblock Training products of experts by minimizing contrastive divergence.
\newblock \emph{Neural computation}, 14\penalty0 (8):\penalty0 1771--1800,
  2002.

\bibitem[Hinton et~al.(2006)Hinton, Osindero, and Teh]{hinton2006fast}
Hinton, Geoffrey~E, Osindero, Simon, and Teh, Yee-Whye.
\newblock A fast learning algorithm for deep belief nets.
\newblock \emph{Neural computation}, 18\penalty0 (7):\penalty0 1527--1554,
  2006.

\bibitem[Honorio et~al.(2009)Honorio, Samaras, Paragios, Goldstein, and
  Ortiz]{honorio2009sparse}
Honorio, Jean, Samaras, Dimitris, Paragios, Nikos, Goldstein, Rita, and Ortiz,
  Luis~E.
\newblock Sparse and locally constant gaussian graphical models.
\newblock In \emph{Advances in Neural Information Processing Systems}, pp.\
  745--753, 2009.

\bibitem[Jia et~al.(2014)Jia, Shelhamer, Donahue, Karayev, Long, Girshick,
  Guadarrama, and Darrell]{jia2014caffe}
Jia, Yangqing, Shelhamer, Evan, Donahue, Jeff, Karayev, Sergey, Long, Jonathan,
  Girshick, Ross, Guadarrama, Sergio, and Darrell, Trevor.
\newblock Caffe: Convolutional architecture for fast feature embedding.
\newblock \emph{arXiv preprint arXiv:1408.5093}, 2014.

\bibitem[Johnson et~al.(1994)Johnson, Kotz, and
  Balakrishnan]{johnson1994continuous}
Johnson, Norman~L, Kotz, Samuel, and Balakrishnan, Narayanaswamy.
\newblock Continuous univariate distributions, vol. 1-2, 1994.

\bibitem[Jordan et~al.(1999)Jordan, Ghahramani, Jaakkola, and
  Saul]{jordan1999introduction}
Jordan, Michael~I, Ghahramani, Zoubin, Jaakkola, Tommi~S, and Saul, Lawrence~K.
\newblock An introduction to variational methods for graphical models.
\newblock \emph{Machine learning}, 37\penalty0 (2):\penalty0 183--233, 1999.

\bibitem[Koller \& Friedman(2009)Koller and Friedman]{koller2009probabilistic}
Koller, Daphne and Friedman, Nir.
\newblock \emph{Probabilistic graphical models: principles and techniques}.
\newblock MIT press, 2009.

\bibitem[Li et~al.(2016)Li, Stevens, Chen, Pu, Gan, and Carin]{li2016shape}
Li, Chunyuan, Stevens, Andrew, Chen, Changyou, Pu, Yunchen, Gan, Zhen, and
  Carin, Lawrence.
\newblock Learning weight uncertainty with stochastic gradient mcmc for shape
  classification.
\newblock In \emph{CVPR}, 2016.

\bibitem[Liao et~al.(2007)Liao, Li, and Carin]{liao2007quadratically}
Liao, Xuejun, Li, Hui, and Carin, Lawrence.
\newblock Quadratically gated mixture of experts for incomplete data
  classification.
\newblock In \emph{Proceedings of the 24th International Conference on Machine
  learning}, pp.\  553--560. ACM, 2007.

\bibitem[Liu \& Willsky(2013)Liu and Willsky]{liu2013learning}
Liu, Ying and Willsky, Alan.
\newblock Learning gaussian graphical models with observed or latent fvss.
\newblock In \emph{Advances in Neural Information Processing Systems}, pp.\
  1833--1841, 2013.

\bibitem[Meng et~al.(2014)Meng, Eriksson, and Hero]{meng2014learning}
Meng, Zhaoshi, Eriksson, Brian, and Hero, Al.
\newblock Learning latent variable gaussian graphical models.
\newblock In \emph{Proceedings of the 31st International Conference on Machine
  Learning (ICML-14)}, pp.\  1269--1277, 2014.

\bibitem[Mudholkar \& Hutson(2000)Mudholkar and Hutson]{mudholkar2000epsilon}
Mudholkar, Govind~S and Hutson, Alan~D.
\newblock The epsilon--skew--normal distribution for analyzing near-normal
  data.
\newblock \emph{Journal of Statistical Planning and Inference}, 83\penalty0
  (2):\penalty0 291--309, 2000.

\bibitem[Neal(1992)]{neal1992connectionist}
Neal, Radford~M.
\newblock Connectionist learning of belief networks.
\newblock \emph{Artificial intelligence}, 56\penalty0 (1):\penalty0 71--113,
  1992.

\bibitem[Oh \& Deasy(2014)Oh and Deasy]{oh2014inference}
Oh, Jung~Hun and Deasy, Joseph~O.
\newblock Inference of radio-responsive gene regulatory networks using the
  graphical lasso algorithm.
\newblock \emph{BMC bioinformatics}, 15\penalty0 (Suppl 7):\penalty0 S5, 2014.

\bibitem[Rezende et~al.(2014)Rezende, Mohamed, and
  Wierstra]{rezende2014stochastic}
Rezende, Danilo~Jimenez, Mohamed, Shakir, and Wierstra, Daan.
\newblock Stochastic backpropagation and approximate inference in deep
  generative models.
\newblock In \emph{Proceedings of The 31st International Conference on Machine
  Learning}, pp.\  1278--1286, 2014.

\bibitem[Salakhutdinov \& Hinton(2009)Salakhutdinov and
  Hinton]{salakhutdinov2009deep}
Salakhutdinov, Ruslan and Hinton, Geoffrey~E.
\newblock Deep boltzmann machines.
\newblock In \emph{International Conference on Artificial Intelligence and
  Statistics}, pp.\  448--455, 2009.

\bibitem[Socci et~al.(1998)Socci, Lee, and Sebastian~Seung]{socci1998rectified}
Socci, Nicholas~D, Lee, Daniel~D, and Sebastian~Seung, H.
\newblock The rectified gaussian distribution.
\newblock \emph{Advances in Neural Information Processing Systems}, pp.\
  350--356, 1998.

\bibitem[Su \& Wu(2015{\natexlab{a}})Su and Wu]{su2015convergence}
Su, Qinliang and Wu, Yik-Chung.
\newblock On convergence conditions of gaussian belief propagation.
\newblock \emph{Signal Processing, IEEE Transactions on}, 63\penalty0
  (5):\penalty0 1144--1155, 2015{\natexlab{a}}.

\bibitem[Su \& Wu(2015{\natexlab{b}})Su and Wu]{su2015distributed}
Su, Qinliang and Wu, Yik-Chung.
\newblock Distributed estimation of variance in gaussian graphical model via
  belief propagation: Accuracy analysis and improvement.
\newblock \emph{Signal Processing, IEEE Transactions on}, 63\penalty0
  (23):\penalty0 6258--6271, 2015{\natexlab{b}}.

\bibitem[Tieleman \& Hinton(2012)Tieleman and Hinton]{tieleman2012lecture}
Tieleman, Tijmen and Hinton, Geoffrey.
\newblock Lecture 6.5-rmsprop: Divide the gradient by a running average of its
  recent magnitude.
\newblock \emph{COURSERA: Neural Networks for Machine Learning}, 4, 2012.

\end{thebibliography}
	\bibliographystyle{icml2016}
	
	\newpage
	\appendix
	
		\twocolumn[
		\begin{center}
			\bf{\Large Supplementary of ``Nonlinear Statistical Learning with Truncated Gaussian Graphical Models''\\ \vspace{7mm}}
		\end{center}
		]
		
		\section{Training TGGM for Classification}
		
		Similar to the regression model, the derivatives of ${\mathcal{Q}}(\cdot)$ can be derived as
		\begin{align}
			&\frac{\partial {\mathcal{Q}}}{\partial {\mathbf{W}}_0} \!=\! -\frac{1}{\sigma^2_0}\left({\mathbb{E}}[{\mathbf{H}}|{\mathbf{X}} ] - {\mathbb{E}}[{\mathbf{H}}|{\mathbf{Y}}, {\mathbf{X}}]\right){\mathbf{X}}^T,  \label{grad_W0_reg}\\
			&\frac{\partial {\mathcal{Q}}}{\partial {\mathbf{b}}_0} \!=\! -\frac{1}{\sigma^2_0}\left({\mathbb{E}}[{\mathbf{H}}|{\mathbf{X}} ] - {\mathbb{E}}[{\mathbf{H}}|{\mathbf{Y}}, {\mathbf{X}}]\right){\mathbf{1}}_N,  \\
			&\frac{\partial {\mathcal{Q}}}{\partial {\mathbf{W}}_1}  \!=\! -\Big({\mathbf{W}}_1{\mathbb{E}}\!\left[{\mathbf{H}}{\mathbf{H}}^T|{\mathbf{Y}}, {\mathbf{X}}\right]  \nonumber \\
			&\hspace{17mm} - ({\mathbb{E}}\left[{\mathbf{Z}}|{\mathbf{Y}}, {\mathbf{X}}\right] \!-\! {\mathbf{b}}_1{\mathbf{1}}^T_N){\mathbb{E}}[{\mathbf{H}}^T|{\mathbf{Y}}, {\mathbf{X}}]\Big), \\
			&\frac{\partial {\mathcal{Q}}}{\partial {\mathbf{b}}_1} \!\!=\! - (N{\mathbf{b}}_1 \!\!-\! ({\mathbb{E}}\!\left[{\mathbf{Z}}|{\mathbf{Y}}, {\mathbf{X}}\right] \!-\!\! {\mathbf{W}}_1{\mathbb{E}}[{\mathbf{H}}|{\mathbf{Y}}, {\mathbf{X}}]){\mathbf{1}}_N)
		\end{align}
		where ${\mathbf{Z}}\triangleq [{\mathbf{z}}_1, {\mathbf{z}}_2, \cdots, {\mathbf{z}}_N]$. With the gradients, we can update the model parameters ${\boldsymbol{\Theta}}$ using appropriate optimization algorithms, such as SGD and its variants.
		
		The prior expectation ${\mathbb{E}}[\cdot|{\mathbf{X}}]$ can be computed easily due to $p({\mathbf{H}}|{\mathbf{X}})$ comprising of univariate truncated normals \cite{johnson1994continuous}. For the posterior expectation ${\mathbb{E}}[\cdot|{\mathbf{Y}}, {\mathbf{X}}]$, we resort to the mean-field VB approximation. Define ${\mathbf{S}}=[{\mathbf{s}}_1, {\mathbf{s}}_2, \cdots, {\mathbf{s}}_N]$ with ${\mathbf{s}}_i \triangleq  {\mathbf{T}}_i{\mathbf{z}}_i$. Suppose a fully factorized distribution $q({\mathbf{H}}, {\mathbf{S}})= \prod_{i=1}^N \prod_{k=1}^K q({\mathbf{h}}_i(k)) q({\mathbf{s}}_i(k))$. Then, we minimize the KL-divergence between $q({\mathbf{H}}, {\mathbf{S}})$ and the true posterior $
		p({\mathbf{Y}},{\mathbf{S}}, {\mathbf{H}}| {\mathbf{X}}) = \prod_{i=1}^N {\mathcal{N}}_T\left({\mathbf{h}}_i \left|  {\mathbf{W}}_0{\mathbf{x}}_i+{\mathbf{b}}_0, \sigma_0^2 {\mathbf{I}}_M \right.\right) \times {\mathcal{N}}\! \left({\mathbf{s}}_i \! \left| {\mathbf{T}}_i ({\mathbf{W}}_1{\mathbf{h}}_i \!+\! {\mathbf{b}}_1), {\mathbf{T}}_i{\mathbf{T}}_i^T \right. \!\right) \times \prod_{k\ne y_i} I({\mathbf{s}}_i(k) \ge 0)
		$, with the KL-divergence expressed as
		\begin{align}
			KL &\!=\!\! -\!\sum_{i=1}^N \frac{1}{2\sigma^2_0} \! \left\langle\left\|{\mathbf{h}}_i \!-\! {\mathbf{W}}_0{\mathbf{x}}_i - {\mathbf{b}}_0\right\|^2 \right\rangle_q \!\! + \left\langle I({\mathbf{h}}_i\ge {\mathbf{0}}) \right\rangle_q \nonumber \\
			&\;\; -\!\sum_{i=1}^N \ln Z_i  \!-\! \sum_{i=1}^N  \frac{1}{2} \left\langle \left\|{\mathbf{T}}_i^{-1} {\mathbf{s}}_i \!-\! {\mathbf{W}}_1{\mathbf{h}}_i \!-\! {\mathbf{b}}_1\right\|^2 \right\rangle_q \nonumber \\
			& \;\; \!+\! \sum_{i=1}^N \!\sum_{k\ne y_i} \!\! \left\langle\ln\left({\mathbf{s}}_i(k) \! \ge 0 \right) \right\rangle_q   \! \!-\! \frac{MN\! \ln 2\pi}{2}\!  \!+\! {\mathcal{H}}(q),
		\end{align}
		where $\left\langle\cdot\right\rangle$ means expectation taken w.r.t. $q({\mathbf{H}}, {\mathbf{S}})$; and ${\mathcal{H}}(q)$ is the entropy of $q({\mathbf{H}}, {\mathbf{S}})$. For convenience of presentation, denote ${\mathbf{v}}_i=[{\mathbf{h}}_i^T, {\mathbf{s}}_i^T]^T$ and ${\mathbf{V}}=[{\mathbf{v}}_1, {\mathbf{v}}_2, \cdots, {\mathbf{v}}_N]$. Thereby, $q({\mathbf{H}}, {\mathbf{S}})$ can now be denoted as $q({\mathbf{V}})$. It is known that when all $q({\mathbf{v}}_s(\ell))$ except $(\ell, s) = (k, i)$ are known, the KL-divergence is minimized if $\ln q({\mathbf{v}}_i(k)) = \left\langle \ln p({\mathbf{y}}_i, {\mathbf{v}}_i| {\mathbf{x}}_i) \right\rangle_{\ne (k,i)} + const$ \cite{jordan1999introduction}. Following the similar procedures and arrangements in regression, it can be obtained that
		\begin{align}
			&\ln p({\mathbf{y}}_i, {\mathbf{v}}_i| {\mathbf{x}}_i)  \nonumber \\
			&\;\;=-\frac{1}{2}{\mathbf{P}}_i(k,k){\mathbf{v}}^2_i(k) + \!\!\!\!\! \sum_{k\ne M+y_i} \!\!\!\!\! \ln I({\mathbf{v}}_i(k)\ge 0) + C_3 \nonumber \\
			&\quad\;\;\; + \left({\boldsymbol{\gamma}}_i(k) - {\mathbf{P}}_i(k,-k){\mathbf{v}}_i(-k) \right){\mathbf{v}}_i(k),
		\end{align}
		where $C_3$ represents all terms without reliance on ${\mathbf{v}}_i(k)$; and
		\begin{align}
			{\mathbf{P}}_i & \triangleq  \left[ {\begin{array}{*{20}{c}}
					{\frac{1}{\sigma^2_0}{\mathbf{I}}_M + {\mathbf{W}}_1^T{\mathbf{W}}_1}&{-{\mathbf{W}}_1^T}\\
					{-{\mathbf{W}}_1}&{\left({\mathbf{T}}_i{\mathbf{T}}_i^T\right)^{-1}},
				\end{array}} \right], \\
				{\boldsymbol{\gamma}}_i & \triangleq \left[ {\begin{array}{*{20}{c}}
						{\frac{1}{\sigma_0^2}({\mathbf{W}}_0{\mathbf{x}}_i+{\mathbf{b}}_0)-{\mathbf{W}}_1^T{\mathbf{b}}_1}\\
						{({\mathbf{T}}_i^{-1})^T{\mathbf{b}}_1}.
					\end{array}} \right]. 
				\end{align}
				From the fact $\ln q({\mathbf{v}}_i(k)) = \left\langle \ln p({\mathbf{y}}_i, {\mathbf{v}}_i| {\mathbf{x}}_i) \right\rangle_{\ne (k,i)} + const$, it can be derived that
				\begin{align} \label{post_pdf_class}
					q({\mathbf{v}}_i(k)) \!=\!\! \left\{ \begin{array}{l}\!\!
						{\mathcal{N}}_T\!\left({\mathbf{v}}_i(k)\left|{\boldsymbol{\varsigma}}_i(k), \frac{1}{{\mathbf{P}}_i(k,k)}\right. \!\right), \;\; \text{if}\; k \!\ne\! M \!+\! y_i,
						\\
						\!\! {\mathcal{N}}\!\left({\mathbf{v}}_i(k)\left|{\boldsymbol{\varsigma}}_i(k), \frac{1}{{\mathbf{P}}_i(k,k)}\right.\!\right), \;\;\;\; \text{otherwise},
					\end{array} \right.
				\end{align}
				where ${\boldsymbol{\varsigma}}_i$ is defined as ${\boldsymbol{\varsigma}}_i(k) = \frac{{\boldsymbol{\gamma}}_i(k) - {\mathbf{\tilde P}}_i(k,:)\left\langle{\mathbf{v}}_i\right\rangle_q}{{\mathbf{P}}_i(k,k)}$ with ${\mathbf{\tilde P}}_i = {\mathbf{P}}_i - \text{diag}({\mathbf{P}}_{i})$. From the distribution $q({\mathbf{v}}_i(k))$, the expectation ${\mathbb{E}}[{\mathbf{v}}_i(k)|{\mathbf{y}}_i, {\mathbf{x}}_i]$ and variance $\text{Var}[{\mathbf{v}}_i(k)|{\mathbf{y}}_i, {\mathbf{x}}_i]$ using the truncated normal properties. With the fact ${\mathbf{v}}_i = [{\mathbf{h}}_i^T, {\mathbf{s}}_i^T]^T$, the expectations ${\mathbb{E}}[{\mathbf{H}}|{\mathbf{Y}}, {\mathbf{X}} ]$, ${\mathbb{E}}[{\mathbf{H}}{\mathbf{H}}|{\mathbf{Y}}, {\mathbf{X}} ]$ and ${\mathbb{E}}[{\mathbf{S}}|{\mathbf{Y}}, {\mathbf{X}}]$ required in the gradient computation can be obtained directly. For ${\mathbf{s}}_i\triangleq {\mathbf{T}}_i{\mathbf{z}}_i$, we have ${\mathbb{E}}\left[{\mathbf{z}}_i|{\mathbf{y}}_i, {\mathbf{x}}_i\right] = {\mathbf{T}}_i^{-1}{\mathbb{E}}\left[{\mathbf{s}}_i|{\mathbf{y}}_i, {\mathbf{x}}_i\right]$, and thus ${\mathbb{E}}[{\mathbf{Z}}|{\mathbf{Y}}, {\mathbf{X}}]$ can be computed easily.
				
				\section{Training Deep TGGM}
				The training algorithms for deep regression and classification TGGMs are almost the same, thus we only present that for classification only. With similar transformation in single layer model, we can represent deep classification TGGM as 
				$p({\mathbf{Y}}, {\mathbf{S}}, {\mathbf{H}}|{\mathbf{X}}; {\boldsymbol{\Theta}})=\prod_{i=1}^N p\left({\mathbf{h}}_i \left|  {\mathbf{x}}_i \right.\right)  \times {\mathcal{N}}\! \left({\mathbf{s}}_i \! \left| {\mathbf{T}}_i ({\mathbf{W}}_2{\mathbf{h}}_i \!+\! {\mathbf{b}}_2), {\mathbf{T}}_i{\mathbf{T}}_i^T \right. \!\right) 
				\times \prod_{k\ne y_i} I({\mathbf{s}}_i(k) \ge 0) $, where $p({\mathbf{h}}_i|{\mathbf{x}}_i)$ is truncated normal distribution defined as $p({\mathbf{h}}_i|{\mathbf{x}}_i) \triangleq \frac{1}{Z_i} \exp\{ -\frac{\|{\mathbf{h}}_i^{(1)} \!-\! {\mathbf{W}}_0{\mathbf{x}}_i \!-\! {\mathbf{b}}_0\|^2}{2\sigma^2_0} -\frac{\|{\mathbf{h}}_i^{(2)} \!-\! {\mathbf{W}}_1{\mathbf{h}}_i^{(1)} \!-\! {\mathbf{b}}_1\|^2}{2\sigma^2_0} \}I\left\{{\mathbf{h}}_i \!\ge\! {\mathbf{0}} \right\}$ with ${\mathbf{h}}_i \triangleq [{\mathbf{h}}_i^{(1)T}, {\mathbf{h}}_i^{(2)T}]^T$ and ${\mathbf{H}} \triangleq [{\mathbf{h}}_1, {\mathbf{h}}_2, \cdots, {\mathbf{h}}_N]$. The lengths of ${\mathbf{h}}_i^{(1)}$ and ${\mathbf{h}}_i^{(2)}$ are denoted as $M_1$ and $M_2$, respectively. It can be seen that the whole model is very closely related to TGGM and preserves most of the properties of truncated normal, thus can be trained efficiently similar to above models. The derivatives of ${\mathcal{Q}}$ can be derived as
				\begin{align}
					\frac{\partial {\mathcal{Q}}}{\partial {\mathbf{W}}_0} &\!=\! -\frac{1}{\sigma^2_0} \Big( {\mathbb{E}} [{\mathbf{H}}^{(1)}|{\mathbf{X}}]  -{\mathbb{E}} [{\mathbf{H}}^{(1)}|{\mathbf{Y}}, {\mathbf{X}}]\Big) {\mathbf{X}}^T;   \\
					\frac{\partial {\mathcal{Q}}}{\partial {\mathbf{b}}_0} &\!=\! -\frac{1}{\sigma^2_0} \Big( {\mathbb{E}} [{\mathbf{H}}^{(1)}|{\mathbf{X}}]  -{\mathbb{E}} [{\mathbf{H}}^{(1)}|{\mathbf{Y}}, {\mathbf{X}}]\Big){\mathbf{1}}_N;\\
					\frac{\partial {\mathcal{Q}}}{\partial {\mathbf{W}}_1} &\!=\!\! -\frac{1}{\sigma^2_0} \Big( {\mathbf{b}}_1{\mathbf{1}}^T_N({\mathbb{E}}[{\mathbf{H}}^{(1)T}|{\mathbf{Y}}, {\mathbf{X}}]- {\mathbb{E}}[{\mathbf{H}}^{(1)T}|{\mathbf{X}}]) \nonumber \\
					&\; + {\mathbf{W}}_1\!({\mathbb{E}}[{\mathbf{H}}^{(1)}{\mathbf{H}}^{(1)T}|{\mathbf{Y}}, {\mathbf{X}}] \!-\! {\mathbb{E}}[{\mathbf{H}}^{(1)}{\mathbf{H}}^{(1)T}|{\mathbf{X}}])\nonumber \\
					&\; - {\mathbb{E}}[{\mathbf{H}}^{(2)}{\mathbf{H}}^{(1)T}|{\mathbf{Y}}, {\mathbf{X}}] \!+\! {\mathbb{E}}[{\mathbf{H}}^{(2)}{\mathbf{H}}^{(1)T}|{\mathbf{X}}]\Big)\\
					\frac{\partial {\mathcal{Q}}}{\partial {\mathbf{b}}_1} &\!=\! -\frac{1}{\sigma^2_0}\Big( {\mathbf{W}}_1({\mathbb{E}}[{\mathbf{H}}^{(1)}|{\mathbf{Y}}, {\mathbf{X}}]-{\mathbb{E}}[{\mathbf{H}}^{(1)}|{\mathbf{X}}]){\mathbf{1}}_N \nonumber \\
					&\quad - ({\mathbb{E}}[{\mathbf{H}}^{(2)}|{\mathbf{Y}}, {\mathbf{X}}] -{\mathbb{E}}[{\mathbf{H}}^{(2)}|{\mathbf{X}}]){\mathbf{1}}_N \Big); \\
					\frac{\partial {\mathcal{Q}}}{\partial {\mathbf{W}}_2} &\!=\!\! -{\mathbf{W}}_2{\mathbb{E}}[{\mathbf{H}}^{(2)}{\mathbf{H}}^{(2)T}|{\mathbf{Y}}, {\mathbf{X}}] \!+\! {\mathbb{E}}[ {\mathbf{Z}} {\mathbf{H}}^{(2)T}|{\mathbf{Y}}, {\mathbf{X}}] \nonumber \\
					&\quad - {\mathbf{b}}_2{\mathbf{1}}_N^T{\mathbb{E}}[{\mathbf{H}}^{(2)T}|{\mathbf{Y}}, {\mathbf{X}}]; \label{derivative_L_W2} \\
					\frac{\partial {\mathcal{Q}}}{\partial {\mathbf{b}}_2} &\!=\!\! -{\mathbf{W}}_2{\mathbb{E}}[{\mathbf{H}}^{(2)}|{\mathbf{Y}}, {\mathbf{X}}]{\mathbf{1}}_N \!-\! N{\mathbf{b}}_2 \!+\! {\mathbb{E}}[{\mathbf{Z}}|{\mathbf{Y}}, {\mathbf{X}}]{\mathbf{1}}_N,
				\end{align}
				where ${\mathbf{H}}^{(\ell)} \triangleq [{\mathbf{h}}_1^{(\ell)}, {\mathbf{h}}_2^{(\ell)}, \cdots, {\mathbf{h}}_N^{(\ell)}]$ for $\ell=1,2$. With the gradients, we can update the model parameters ${\boldsymbol{\Theta}}$ using appropriate optimization algorithms, such as SGD and its variants.
				
				In deep models, since the prior is also a multivariate truncated normal, it is expensive to compute the prior expectation ${\mathbb{E}}[\cdot|{\mathbf{X}}]$ analytically as that in one-layer case.  For the efficiency of training, we resort to mean-field VB for the estimation of both prior and posterior expectations ${\mathbb{E}}[\cdot|{\mathbf{X}}]$ and ${\mathbb{E}}[\cdot|{\mathbf{Y}}, {\mathbf{X}}]$. The prior distribution $p({\mathbf{h}}_i|{\mathbf{x}}_i)$ can be equivalently written as
				\begin{align}
					p({\mathbf{h}}_i|{\mathbf{x}}_i) = {\mathcal{N}}_T\left({\mathbf{h}}_i| {\mathbf{Q}}^{-1}{\boldsymbol{\beta}}_i, {\mathbf{Q}}^{-1}\right),
				\end{align}
				where 
				\begin{align}
					{\mathbf{Q}} & \triangleq \frac{1}{\sigma^2_0}\left[ {\begin{array}{*{20}{c}}
							{\mathbf{I}}_{M_1}+{\mathbf{W}}_1^T{\mathbf{W}}_1 & -{\mathbf{W}}_1^T\\
							-{\mathbf{W}}_1 & {\mathbf{I}}_{M_2}
						\end{array}} \right],  \\
						{\boldsymbol{\beta}}_i & \triangleq \frac{1}{\sigma^2_0}\left[ {\begin{array}{*{20}{c}}
								{\mathbf{W}}_0{\mathbf{x}}_i+{\mathbf{b}}_0-{\mathbf{W}}^T_1{\mathbf{b}}_1\\
								{\mathbf{b}}_1
							\end{array}} \right].
						\end{align}
						Suppose a fully factorized distribution $q({\mathbf{h}}_i)=\prod_{k=1}^{M} q({\mathbf{h}}_i(k))$ with $M \triangleq M_1+M_2$. We now minimize the KL-divergence between $q({\mathbf{h}}_i)$ and the true posterior $p({\mathbf{h}}_i|{\mathbf{x}}_i)$. It is known that when all $q({\mathbf{h}}_s(\ell))$ except $(\ell, s) = (k, i)$ are known, the KL-divergence is minimized if $\ln q({\mathbf{h}}_i(k)) = \left\langle \ln p({\mathbf{h}}_i| {\mathbf{x}}_i) \right\rangle_{\ne (k,i)} + const$. By rearranging the terms in $\ln p({\mathbf{h}}_i|{\mathbf{x}}_i)$, it can be easily obtained that $p({\mathbf{h}}_i|{\mathbf{x}}_i) = -\frac{1}{2}{\mathbf{Q}}(k,k){\mathbf{h}}_i^2(k) + ({\boldsymbol{\beta}}_i(k)-{\mathbf{Q}}(k,-k){\mathbf{h}}_i(-k)){\mathbf{h}}_i(k) + \ln I({\mathbf{h}}_i(k)\ge 0) + C_4$. Thus, we have
						\begin{equation}
							q({\mathbf{h}}_i(k)) \!=\! {\mathcal{N}}_T \!\left(\! {\mathbf{h}}_i(k)\left| \frac{{\boldsymbol{\beta}}_i(k) - {\mathbf{ Q}}_0(k,:)\left\langle {\mathbf{h}}_i\right\rangle_q}{{\mathbf{Q}}(k,k)}, \frac{1}{{\mathbf{Q}}(k,k)} \right.\!\right),
						\end{equation}
						where ${\mathbf{Q}}_0 \triangleq {\mathbf{Q}} - \text{diag}({\mathbf{Q}})$. From the univariate truncated normal $q({\mathbf{h}}_i)$ , we can estimate the prior expectation ${\mathbb{E}}[{\mathbf{H}} |{\mathbf{X}}]$ easily.
						
						For posterior expectation ${\mathbb{E}}[\cdot|{\mathbf{Y}}, {\mathbf{X}}]$, we also suppose a fully factorized distribution $q({\mathbf{v}}_i)$ with ${\mathbf{v}}_i = [{\mathbf{h}}_i^T, {\mathbf{s}}_i^T]^T$ and then minimize the KL-divergence. First, we express the log-likelihood as
						\begin{align}
							&\ln p({\mathbf{y}}_i, {\mathbf{v}}_i|{\mathbf{x}}_i) \nonumber \\
							&\;\; = -\frac{1}{2}{\mathbf{P}}_i(k,k){\mathbf{v}}_i^2(k) +\!\! \sum\limits_{k\ne M+y_i} \ln I({\mathbf{v}}_i(k) \ge 0) \nonumber \\
							&\quad\;\; + ({\boldsymbol{\gamma}}_i(k) - {\mathbf{P}}_i(k,-k){\mathbf{v}}_i(-k)){\mathbf{v}}_i(k) + C_4,
						\end{align}
						where in deep models ${\mathbf{P}}_i$ and ${\boldsymbol{\gamma}}_i$ are defined as
						\begin{align} \label{P_deep_class}
							{\mathbf{P}}_i &\!=\!  \frac{1}{\sigma^2_0} \!\! \left[\!\! {\begin{array}{*{20}{c}}
									{\mathbf{I}}_{M_1} \!\!+\!\! {\mathbf{W}}_1^T{\mathbf{W}}_1 & -{\mathbf{W}}_1^T & {\mathbf{0}}\\
									-{\mathbf{W}}_1 & {\mathbf{I}}_{M_2} \!\!+\!\! {\mathbf{W}}_2^T{\mathbf{W}}_2 & -{\mathbf{W}}_2^T{\mathbf{T}}_i^{-1}\\
									0 & -{\mathbf{T}}_i^{-1T}{\mathbf{W}}_2 & ({\mathbf{T}}_i{\mathbf{T}}_i^T)^{-1}
								\end{array}} \!\!\! \right]; \\
								{\boldsymbol{\gamma}}_i &\!=\! \frac{1}{\sigma^2_0}\left[ {\begin{array}{*{20}{c}}
										{\mathbf{W}}_0{\mathbf{x}}_i+{\mathbf{b}}_0-{\mathbf{W}}_1^T{\mathbf{b}}_1\\
										{\mathbf{b}}_1-{\mathbf{W}}_2^T{\mathbf{b}}_2\\
										{\mathbf{T}}_i^{-1T}{\mathbf{b}}_2
									\end{array}} \right].
								\end{align}
								Then, it can be known that the KL-divergence is minimized with $q({\mathbf{v}}_i(k))$ being the same form as \eqref{post_pdf_class}. The only difference are the expressions of precision matrix ${\mathbf{P}}_i$ and linear vector ${\boldsymbol{\gamma}}_i$. With the factorized truncated normal distribution, the posterior expectations ${\mathbb{E}}[{\mathbf{H}}|{\mathbf{Y}}, {\mathbf{X}}]$ and covariance ${\mathbb{E}}[{\mathbf{H}}{\mathbf{H}}^T|{\mathbf{Y}}, {\mathbf{X}}]$ can be estimated easily using truncated normal properties.
	
\end{document}